\documentclass{article}
\usepackage{ijcai26}

\usepackage{times}
\usepackage{soul}
\usepackage{url}
\usepackage[hidelinks]{hyperref}
\usepackage[utf8]{inputenc}
\usepackage[small]{caption}
\usepackage{graphicx}
\usepackage{amsmath}
\usepackage{amsthm}
\usepackage{booktabs}
\usepackage{algorithm}
\usepackage{algorithmic}
\usepackage[switch]{lineno}
\usepackage{amsmath,amssymb,amsfonts}
\usepackage{enumitem}
\usepackage{textcomp}
\usepackage{colortbl}
\usepackage{xcolor}
\usepackage[most]{tcolorbox}
\usepackage{subcaption}
\usepackage{multirow}
\usepackage{pifont}
\usepackage{tabularx}

\newtheorem{theorem}{Theorem}

\title{\emph{A$_3$B$_2$}: Adaptive Asymmetric Adapter for Alleviating Branch Bias in Vision-Language Image Classification with Few-Shot Learning}

\author{
Yiyun Zhou$^1$
\and
Zhonghua Jiang$^1$\and
Wenkang Han$^1$\and
Kunxi Li$^1$\and
Mingjing Xu$^{2}$\and
Chang Yao$^{1}$\And
Jingyuan Chen$^1$\thanks{Corresponding author.}\\
\affiliations
$^1$Zhejiang University\\
$^2$Swansea University\\
\emails
\{yiyunzhou, jiangzhonghua, wenkangh, kunxili, changy, jingyuanchen\}@zju.edu.cn,
mingjing.xu@swansea.ac.uk
}

\begin{document}
\theoremstyle{definition}                
\newtheorem{assumption}{Assumption}[section]   
\newtheorem{definition}{Definition}

\theoremstyle{plain}                     
\newtheorem{lemma}[theorem]{Lemma}       
\newtheorem{proposition}[theorem]{Proposition}
\newtheorem{corollary}[theorem]{Corollary}

\theoremstyle{remark}                    
\newtheorem{remark}{Remark}[section]

\maketitle

\begin{abstract}

Efficient transfer learning methods for large-scale vision–language models (\textit{e.g.}, CLIP) enable strong few-shot transfer, yet existing adaptation methods follow a fixed fine-tuning paradigm that implicitly assumes a uniform importance of the image and text branches, which has not been systematically studied in image classification. Through extensive analysis, we reveal a \emph{Branch Bias} issue in vision-language image classification: adapting the image encoder does not always improve performance under out-of-distribution settings. Motivated by this observation, we propose \textbf{\emph{A$_3$B$_2$}}, an \textbf{\underline{A}}daptive \textbf{\underline{A}}symmetric \textbf{\underline{A}}dapter that alleviates \textbf{\underline{B}}ranch \textbf{\underline{B}}ias in few-shot learning. \emph{A$_3$B$_2$} introduces \textit{Uncertainty-Aware Adapter Dampening} (\emph{UAAD}), which automatically suppresses image-branch adaptation when prediction uncertainty is high, enabling soft and data-driven control without manual intervention. Architecturally, \emph{A$_3$B$_2$} adopts a lightweight asymmetric design inspired by mixture-of-experts with \textit{Load Balancing Regularization}. Extensive experiments on three few-shot image classification tasks across 11 datasets demonstrate that \emph{A$_3$B$_2$} consistently outperforms 11 competitive prompt- and adapter-based baselines.

\end{abstract}

\vspace{-0.2cm}
\section{Introduction}
\label{sec:introduction}
\vspace{-0.1cm}

Large-scale Vision-Language Models (VLMs)~\cite{radford2021learning,alayrac2022flamingo,li2023scaling,li2025flowmm,jiang2026acckv,jiang2025purekv,zhou2026collaborative} with dual branches are foundational for downstream transfer learning tasks due to their strong generalization. CLIP~\cite{radford2021learning}, trained on 400 million image-text pairs, aligns visual and textual modalities in a shared feature space via separate encoders and contrastive learning. However, full fine-tuning VLMs for specific tasks is computationally expensive and prone to overfitting in few-shot settings~\cite{zhou2025cuff}.

To address the aforementioned issues, efficient transfer learning methods have been proposed~\cite{zhou2022learning,yang2024mma,guo2025mmrl,li2025mergenet,guo2025mmrl++}, enabling VLMs to adapt to different downstream tasks with minimal additional parameters. Previous approaches have mainly focused on prompt tuning and adapter-based methods. Prompt tuning aims to address the limitation of VLMs that require manually designed text prompts, which requires expert knowledge and is time-consuming. For instance, in the OxfordPets~\cite{parkhi2012cats} dataset, CLIP's text encoder computes class embeddings using the prompt ``\texttt{A photo of a [CLASS], a type of pet.}''. These embeddings are then matched with the outputs by the image encoder, and the most similar one is selected to predict the image category. Most prior prompt tuning methods~\cite{zhou2022learning,guo2025mmrl} introduce learnable continuous vectors into the text encoder to dynamically model prompts, optimizing them during training while freezing other parameters, thereby enabling adaptation to specific datasets. In contrast, adapter-based methods~\cite{yang2024mma} add lightweight, architecture-agnostic modules, making them easily integrable into different VLMs and allowing the models to adapt to potential downstream task spaces.

Although prior methods have achieved impressive few-shot performance in data and model adaptation, most  CLIP-based methods follow a fixed fine-tuning paradigm: prompt tuning methods optimize the text encoder, while adapter-based methods are attached to the image encoder~\cite{li2024vision,zhang2024vision}. This raises a compelling research question: \textit{\textbf{which matters more in different few-shot tasks, CLIP's text or image encoder?}}~\footnote{We refer to this as the CLIP's \emph{Branch Bias} issue.} Recent studies have sparked interest in this issue. \citeauthor{yang2025language} train only the image encoder while keeping the text encoder fixed, achieving surprising results on compositional understanding and long-text description tasks. \citeauthor{gong2025kernel} propose a kernel-based method, significantly improving the image encoder's performance on zero-shot target recognition, fine-grained spatial reasoning, and localization tasks. \citeauthor{fu2025hidden} find that in vision-centered benchmarks (\textit{e.g.}, depth estimation, correspondence), VLMs perform worse than their image encoders, with performance nearing random levels. In open-vocabulary semantic segmentation tasks, most past methods~\cite{li2022language,liang2023open,xu2023side} freeze the text encoder to preserve generalization, while recent studies~\cite{peng2025understanding} show that jointly fine-tuning both encoders leads to better segmentation results. However, \textbf{there is no consensus on the \emph{Branch Bias} issue in the general tasks of image classification}.

To address gaps in prior work, we conduct extensive exploratory experiments on three visual tasks: base-to-novel generalization, cross-dataset evaluation, and domain generalization. Crucially, our insights reveal a significant phenomenon: \textbf{fine-tuning the image encoder does not always yield performance gains and may even degrade robustness on out-of-distribution (OOD) tasks}. While a straightforward response to this discovery would be to rigidly lock the image encoder for OOD scenarios, such a hard-coded approach is cumbersome and impractical, as it requires prior knowledge of the target data distribution to manually toggle the adapters. To resolve this dilemma, we propose a novel \textit{\textbf{A}daptive \textbf{A}symmetric \textbf{A}dapter for alleviating \textbf{B}ranch \textbf{B}ias} (\textbf{\emph{A$_3$B$_2$}}) in vision-language image classification. Unlike rigid switching mechanisms, \emph{A$_3$B$_2$} achieves adaptability through a uncertainty-aware optimization mechanism. Specifically, we introduce \textit{Uncertainty-Aware Adapter Dampening} (\emph{UAAD}). This mechanism dynamically monitors prediction confidence; when \textbf{the model encounters high-uncertainty samples (indicative of potential OOD data), it automatically suppresses the magnitude of the image adapter's contribution} via a regularization penalty. This soft-constraint approach effectively operationalizes our insights without requiring manual intervention. Structurally, inspired by the Mixture of Experts (MoE) design~\cite{fedus2022switch,xue2022go,tian2024hydralora,gao2024higher,zhou2025cola}, \emph{A$_3$B$_2$} consists of a single dimensionality-reduction matrix and multiple dimensionality-expansion matrices, further enhanced by a \textit{Load Balancing Regularization} to ensure diverse expert utilization~\cite{shazeer2017outrageously}.

Our contributions are summarized as follows:
\begin{itemize}[leftmargin=*]
    \item{We systematically explore the \textit{Branch Bias} issue in image classification for \textbf{transformer-based CLIP models}. Three insights discovered help build the optimal adapter structure for vision-language image classification tasks.}
    \item{We propose a task-adaptive, structure-asymmetric adapter to alleviate \textit{Branch Bias} in vision-language image classification. \textbf{Empirical evidence, theoretical support for the asymmetric structure, and theoretical guarantees of the method} are provided in Appendix \ref{sec:fixed} and \ref{sec:guarantee}, respectively.}
    \item{Evaluations on three few-shot tasks show that, across 11 image classification datasets, \emph{A$_3$B$_2$} consistently achieves leading performance compared to 11 competitive baselines. The experiments also confirm the generality of our insights.}
\end{itemize}

\vspace{-0.1cm}
\section{Related Work}
\label{sec:related_work}
\vspace{-0.1cm}

\subsection{CLIP’s Branch Bias in Various Tasks}

Recent studies have sparked significant interest in the CLIP's \emph{Branch Bias} issue, particularly in tasks requiring fine-grained visual discrimination. For compositional and spatial tasks, the image branch often requires more adaptation. \citeauthor{yang2025language} demonstrate that training only the image encoder while keeping the text encoder fixed yields superior performance in compositional understanding and long-text description. Similarly, \citeauthor{gong2025kernel} propose a kernel-based alignment method that focuses on the image encoder, significantly improving zero-shot target recognition and localization. \citeauthor{fu2025hidden} further reveal that in strictly vision-centered benchmarks (\textit{e.g.}, depth estimation), full VLMs can paradoxically perform worse than their frozen image encoders, suggesting textual interference.
In contrast, for open-vocabulary semantic segmentation, the consensus has shifted. While earlier methods~\cite{li2022language,liang2023open,xu2023side} froze the text encoder to preserve generalization, recent work by \citeauthor{peng2025understanding} argues that jointly fine-tuning both encoders is essential for aligning pixel-level features with semantic concepts. Despite these diverse findings in structured prediction and reasoning tasks, the manifestation of this \emph{Branch Bias} in general image classification remains largely underexplored.

\subsection{Efficient Transfer Learning in Vision-Language Image Classification}

Prompt learning methods enhance the adaptability of Vision-Language Models (VLMs) by dynamically adjusting pretrained representations. CoOp~\cite{radford2021learning} introduces the learnable prompt vector, replacing fixed templates with optimized contexts to improve flexibility, though this sacrifices CLIP’s zero-shot generalization. CoCoOp~\cite{zhou2022conditional} addresses this limitation by incorporating dynamic instance-conditioned prompts, improving generalization to unseen categories. KgCoOp~\cite{yao2023visual} further constrains the semantic gap between learned and handcrafted prompts, preserving general textual knowledge and enhancing model generalization. RPO~\cite{lee2023read} employs a masked attention mechanism to freeze the pretrained representations, reducing internal representation changes and improving model stability, particularly in data-scarce scenarios. MaPLe~\cite{khattak2023maple} achieves cross-modal alignment by coupling visual-text prompts, enhancing adaptability to new categories and unseen domains. TCP~\cite{yao2024tcp} innovatively embeds category semantics into prompt tokens, strengthening generalization for new categories and few-shot settings. MMRL~\cite{guo2025mmrl} proposes a shared learnable representation space to facilitate efficient visual-text interaction, with regularization to prevent overfitting, especially suited for few-shot learning and task adaptation. MMRL++~\cite{guo2025mmrl++} further optimizes this framework by employing a parameter-efficient low-rank alignment and progressive feature composition, which jointly facilitate robust generalization and training stability with minimal overhead.

Recent adapter-style methods provide another effective approach for model adaptation, typically integrating lightweight modules into VLMs for downstream task. CLIP-Adapter~\cite{gao2024clip} introduces a residual MLP layer at the image encoder’s end, significantly improving performance in visual tasks. MMA~\cite{yang2024mma} proposes a multimodal adapter method, aggregating image and text features into a shared space, enhancing consistency between modalities and allowing cross-modal gradient flow, thereby improving the model's overall adaptability and task transfer capability. However, none of these methods have systematically investigated the \emph{Branch Bias} in vision-language image classification, typically adhering to static adaptation paradigms that fail to dynamically balance the dual encoders. Consistent with previous research~\cite{zhou2022learning,zhou2022conditional,guo2025mmrl++}, we use CLIP (transformer-based ViT-B/16 backbone) as the base model for our method.

\section{Methodology}
\label{sec:method}

\setlength{\abovedisplayskip}{3.5pt}
\setlength{\belowdisplayskip}{3.5pt}

\subsection{\textit{Preliminary}}
\label{sec:preliminary}

The CLIP model consists of an image encoder $\mathcal{V}$ and a text encoder $\mathcal{T}$, encoding images and corresponding texts.

\paragraph{Image Encoder}

The image encoder $\mathcal{V}$ consists of $L$ Transformer blocks~\cite{vaswani2017attention}, denoted as $\{\mathcal{V}_i\}_{i=1}^L$. The input image $I \in \mathbb{R}^{H\times W \times 3}$ is divided into $M$ patches, which are then embedded to obtain the initial features $x_0$:
\begin{align}
x_0 = \texttt{Patch\ Embedding}(I),
\end{align}
where $x_0 \in \mathbb{R}^{M\times d_v}$ ($d_v$ is the image embedding dimension).

The patch embeddings $x_{i-1}$, along with a learnable class (CLS) token $c_{i-1}$, serve as inputs to the next Transformer block $\mathcal{V}_i$ and are processed sequentially:
\begin{align}
[c_i, x_i] = \mathcal{V}_i([c_{i-1}, x_{i-1}]), \quad i = 1, 2, \dots, L. \label{eq:image_inert}
\end{align}

Finally, a patch projection layer maps the CLS token $c_L$ from the last Transformer block $\mathcal{V}_L$ into a shared Vision-Language (V-L) latent space to obtain the final image representation $x$:
\begin{align}
x = \texttt{Patch\ Projection}(c_L),
\end{align}
where $x \in \mathbb{R}^d$, and $d$ is the feature dimension of the shared V-L space.

\paragraph{Text Encoder}

For a given text $T$ (\textit{e.g.}, ``\texttt{A photo of a [CLASS].}''), it is first tokenized and converted into word embeddings $w_0$ using a Text Embedding layer:
\begin{align}
[w_0^j]_{j=1}^N = \texttt{Text\ Embedding}(T),
\end{align}
where $w_0 \in \mathbb{R}^{N\times d_t}$, with $N$ as the token length and $d_t$ as the text embedding dimension.

At each layer, $w_{i-1}$ is sequentially processed by the next Transformer block $\mathcal{T}_i$ to extract text features $w_i$:
\begin{align}
[w_i^j]_{j=1}^N = \mathcal{T}_i([w_{i-1}^j]_{j=1}^N), \quad i = 1, 2, \dots, L. \label{eq:text_inert}
\end{align}

The final text representation $w$ is obtained by projecting the last token output from the final Transformer block $\mathcal{T}_L$ into the shared V-L space via a Text Projection layer:
\begin{align}
w = \texttt{Text\ Projection}(w_L^N),
\end{align}
where $w \in \mathbb{R}^d$.

\paragraph{Zero-shot Classification}

Given the image representation $x$ and text representations $\{w_c\}_{c=1}^{N_c}$, where $N_c$ is the number of classes, CLIP computes the cosine similarity between them:
\begin{align}
\texttt{cos}(x, w_c) = \frac{x \cdot w_c}{|x||w_c|},
\end{align}
where $|\cdot|$ denotes the $L_2$ norm. The classification probability for each class $c$ is then predicted as:
\begin{align}
\textbf{p}(\hat{y} = c \mid x) = \frac{\exp(\texttt{cos}(x, w_c)/\tau)}{\sum_{i=1}^{N_c} \exp(\texttt{cos}(x, w_i)/\tau)},
\end{align}
where $\tau$ is a learnable temperature parameter in CLIP. The class with the highest probability is selected as the final prediction.

\begin{figure}[ht]
    \centering
    \begin{subfigure}[b]{\linewidth}
        \includegraphics[width=\linewidth]{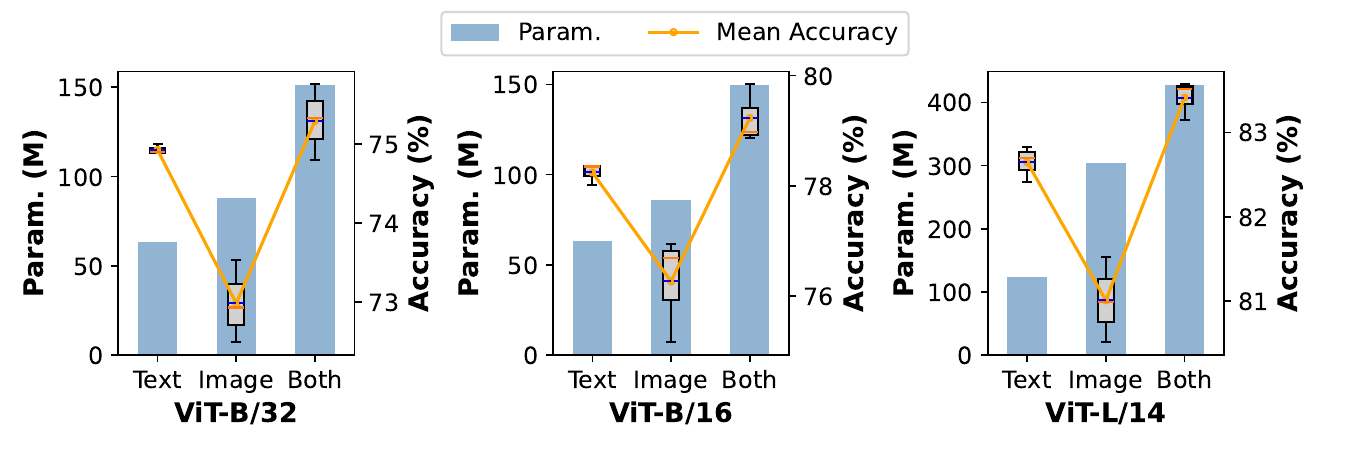}
        \vspace{-0.7cm}
        \caption{Base-to-Novel Generalization}
        \label{fig:base_to_novel}
    \end{subfigure}

    \begin{subfigure}[b]{\linewidth}
        \includegraphics[width=\linewidth]{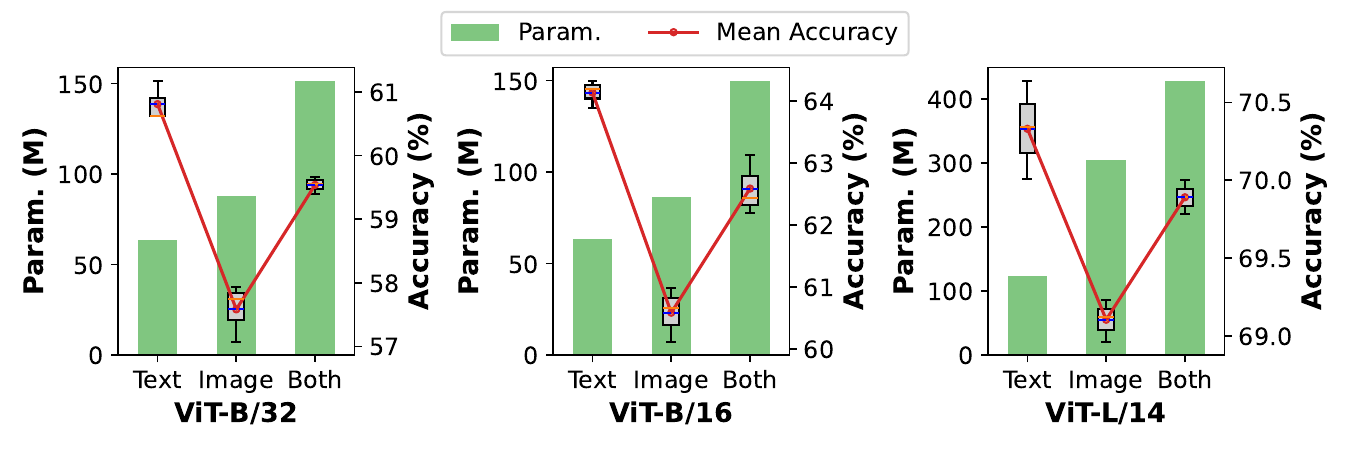}
        \vspace{-0.7cm}
        \caption{Cross-Dataset Evaluation}
        \label{fig:cross_dataset}
    \end{subfigure}

    \begin{subfigure}[b]{\linewidth}
        \includegraphics[width=\linewidth]{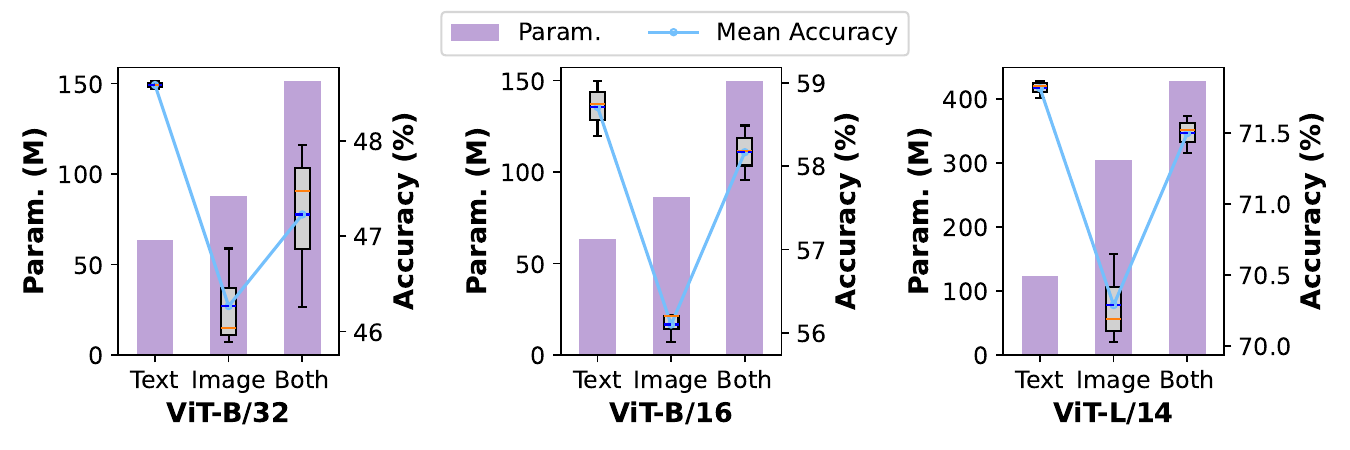}
        \vspace{-0.7cm}
        \caption{Domain Generalization}
        \label{fig:domain_generalization}
    \end{subfigure}
    \vspace{-0.7cm}
    \caption{The average performance of text or image adapters on three few-shot image classification tasks with three pretrained transformer-based CLIP models (empirical evidence from Branch Bias). The results on each dataset can be found in the Appendix~\ref{sec:motivation}.}
    \label{fig:motivation}
    \vspace{-0.4cm}
\end{figure}

\begin{figure*}[ht]
\centering
    \includegraphics[width=\linewidth]{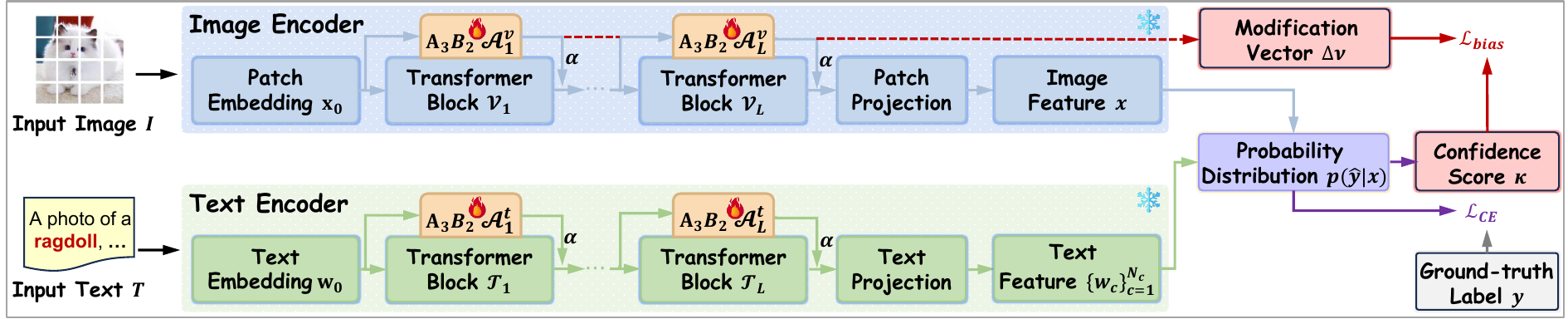}
    \vspace{-0.7cm}
    \textbf{\caption{Overview of the proposed \emph{A$_3$B$_2$} architecture. The asymmetric adapters are integrated into each Transformer layer of the CLIP.}
    \label{fig:method}}
    \vspace{-0.4cm}
\end{figure*}

\begin{figure}[ht]
\centering
    \includegraphics[width=\linewidth]{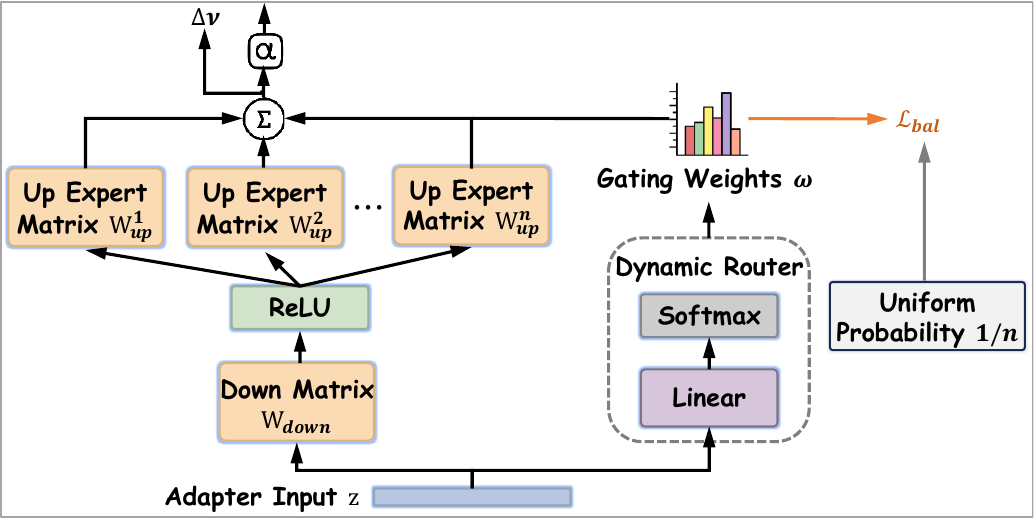}
    \vspace{-0.7cm}
    \textbf{\caption{Detailed structure of the \emph{A$_3$B$_2$} adapter. The module consists of a shared down-projection layer and a dynamic router that adaptively weights multiple up-projection experts.}
    \label{fig:adapter}}
    \vspace{-0.4cm}
\end{figure}

\subsection{\emph{A$_3$B$_2$}: \textit{Adaptive Asymmetric Adapter}}
\label{sec:a3}
\subsubsection{\textit{Experimental Observations on Branch Bias}}
We conduct extensive exploratory experiments to investigate CLIP's \emph{Branch Bias} in image classification tasks. Specifically, we apply 11 diverse image datasets across three widely used image classification tasks (base-to-novel generalization, cross-dataset evaluation, and domain generalization). By fine-tuning a basic adapter (a pair of dimensionality reduction and expansion matrices) on each layer of the image or text encoders in three different CLIP variants (ViT-B/32, ViT-B/16, ViT-L/14), we evaluate performance across different tasks. Fig.~\ref{fig:motivation} shows the average performance across these tasks, along with the parameter sizes of the image and text encoders in various CLIP variants. Further details (experimental setup and additional results) are available in the Appendix. We summarize the following three insights:

\paragraph{\texttt{Insight I}} In various tasks, fine-tuning additional parameters on the image encoder results in lower performance gains compared to the text encoder.

\paragraph{\texttt{Insight II}} In in-distribution tasks (\textit{e.g.}, base-to-novel generalization), fine-tuning adapters on both encoders  achieves the best performance.

\paragraph{\texttt{Insight III}} In out-of-distribution  tasks (\textit{e.g.}, cross-dataset evaluation and domain generalization), fine-tuning additional parameters on the image encoder may harm the transferability of VLMs on non-distribution data.

\subsubsection{\textit{Task-adaptive and Structure-asymmetric Adapter}}
\label{sec:macro}

Based on the insights above, we propose an asymmetric architecture where adapters are always active on the text encoder (\texttt{Insight I}), while the image encoder's adaptation is dynamically regulated (\texttt{Insight II \& III}).

As shown in Fig.~\ref{fig:method}, we integrate the \emph{A$_3$B$_2$} adapter (details in Fig.~\ref{fig:adapter}) into the CLIP backbone. For the text encoder, we modify Eq.~\ref{eq:text_inert} by adding the adapter $\mathcal{A}^t$ to each layer:

\begin{align}
\begin{cases}
[w_i^j]_{j=1}^N = \mathcal{T}_i([w_{i-1}^j]_{j=1}^N) + \alpha \underline{\mathcal{A}_i^t([w_{i-1}^j]_{j=1}^N)}, \\
\quad \quad \quad \quad \quad \quad \quad \quad \quad \quad \quad \quad \quad i = 1, \dots, L,\\
\mathcal{A}_i^t(z=[w_{i-1}^j]_{j=1}^N) = \sum_k^n \overset{t}{\mathfrak{w}}_i^k {W_{up}^{t}}_i^k \\
\quad \quad \quad \quad \quad \quad \quad \quad \quad \left(\texttt{ReLU}({W_{down}^{t}}_i(z))\right)
\end{cases}
\end{align}
where \underline{underline} indicates trainable modules, and $\alpha$ controls the adaptation strength. $\texttt{ReLU}(\cdot)$ is the activation function. ${W_{down}^{t}}_i\in \mathbb{R}^{d_{ht} \times r}$ and ${W_{up}^{t}}_i^k \in \mathbb{R}^{r\times d_{ht}}$, where $d_{ht}$ represents the hidden layer dimension of the text encoder and rank $r \ll d_{ht}$, represent the ``Down'' matrix and \textit{k}-th ``Up'' expert matrix in text encoder layer $\mathcal{T}_i$, respectively. Additionally, a dynamic router computes the expert weights:
\begin{align}
\overset{t}{\mathfrak{w}}_i = \texttt{softmax}({\overset{t}{W}}_i^g z), \label{eq:router_text}
\end{align}
where ${\overset{t}{W}}_i^g\in \mathbb{R}^{d_{ht}\times n}$, $n$ is the number of experts.

\begin{table*}[!t]
\centering
\resizebox{0.95\textwidth}{!}{%
\begin{tabular}{@{}c|c|ccccccccccc|c@{}}
\toprule[1.5pt]
\midrule
\textbf{Dataset} & \textbf{Setting} & \textbf{CLIP} & \textbf{CoOp} & \textbf{CoCoOp} & \textbf{KgCoOp} & \textbf{RPO}   & \textbf{MaPLe} & \textbf{CLIP-Adapter} & \textbf{TCP} & \textbf{MMA}   & \textbf{MMRL} & \textbf{MMRL++} & \textbf{\emph{A$_3$B$_2$}}                          \\ \midrule 
                 & Base             & 72.40         & 76.18         & 75.63           & 76.08           & 74.53          & 76.73          & 74.10                 & 75.39        & \underline{77.19}    & 77.00     & 76.87     & \multicolumn{1}{c}{\cellcolor[HTML]{FFCCC9}\textbf{77.37}} \\
\textbf{ImageNet}         & Novel            & 68.09         & 70.47         & 70.45           & 70.59     & 70.26          & 70.49          & 68.89                 & 70.21        & 69.63          & 70.70 & \underline{70.72} & \multicolumn{1}{c}{\cellcolor[HTML]{FFCCC9}\textbf{70.83}}          \\
                 & HM               & 70.18         & 73.21         & 72.95           & 73.23           & 72.33          & 73.48          & 71.40                 & 72.08        & 73.22          & \underline{73.72}  & 73.67  & \multicolumn{1}{c}{\cellcolor[HTML]{FFCCC9}\textbf{73.96}} \\ \midrule
                 & Base             & 97.16         & 98.36   & 97.59           & 98.15           & 97.14          & 98.26          & 97.29                 & 97.61        & \underline{98.54} & 98.19   & 98.26       & \multicolumn{1}{c}{\cellcolor[HTML]{FFCCC9}\textbf{98.73}} \\
\textbf{Caltech101}       & Novel            & 94.21         & 94.18         & \underline{94.69}  & 94.65           & 94.32          & 94.14          & 94.10                 & 94.54        & 93.63          & 94.07   & 94.36       & \multicolumn{1}{c}{\cellcolor[HTML]{FFCCC9}\textbf{95.17}}    \\
                 & HM               & 95.66         & 96.22         & 96.12           & \underline{96.37}     & 95.71          & 82.54          & 95.67                 & 96.05        & 96.02          & 96.09     & 96.27     & \multicolumn{1}{c}{\cellcolor[HTML]{FFCCC9}\textbf{96.92}} \\ \midrule
                 & Base             & 91.26         & 95.27         & 95.39           & 94.84           & 88.76          & 95.21          & 94.22                 & 94.79        & 95.20          & \underline{95.41} & 95.16   & \multicolumn{1}{c}{\cellcolor[HTML]{FFCCC9}\textbf{95.52}} \\
\textbf{OxfordPets}       & Novel            & 97.15         & 97.78         & 97.32           & 97.43           & 96.66          & \textbf{97.93}    & 97.05                 & 97.59        & 97.59          & 97.26     & 97.60      & \multicolumn{1}{c}{\cellcolor[HTML]{FFCCC9}\underline{97.83}} \\
                 & HM               & 94.11         & 96.51         & 96.35           & 96.12           & 92.54          & \underline{96.55}    & 95.61                 & 96.17        & 96.38          & 96.33    & 96.36      & \multicolumn{1}{c}{\cellcolor[HTML]{FFCCC9}\textbf{96.66}} \\ \midrule
                 & Base             & 63.77         & 74.89         & 69.28           & 71.93           & 66.18          & 76.17          & 65.93                 & 70.47        & \underline{78.95}    & 75.55   & 75.60       & \multicolumn{1}{c}{\cellcolor[HTML]{FFCCC9}\textbf{81.96}} \\
\textbf{StanfordCars}     & Novel            & 74.96         & 71.70         & 73.83           & 73.35           & \textbf{75.65} & 73.33          & 75.44           & 75.40        & 72.53          & 74.99   & 74.83       & \multicolumn{1}{c}{\cellcolor[HTML]{FFCCC9}\underline{75.58}}          \\
                 & HM               & 68.91         & 73.26         & 71.48           & 72.63           & 70.60          & 74.72          & 70.37                 & 72.85        & \underline{75.60}    & 75.27    & 75.21      & \multicolumn{1}{c}{\cellcolor[HTML]{FFCCC9}\textbf{78.64}} \\ \midrule
                 & Base             & 71.73         & 97.12         & 86.29           & 96.20           & 73.82          & 97.09          & 75.18                 & 93.41        & \underline{97.79}    & \underline{97.79}  & 97.68  & \multicolumn{1}{c}{\cellcolor[HTML]{FFCCC9}\textbf{98.14}} \\
\textbf{Flowers102}       & Novel            & \underline{77.45}   & 71.94         & 74.21           & 72.88           & \textbf{78.44} & 72.03          & 75.58                 & 75.25        & 75.96          & 75.13     & 76.59      & \multicolumn{1}{c}{\cellcolor[HTML]{FFCCC9}77.39}          \\
                 & HM               & 74.48         & 82.65         & 79.80           & 82.93           & 76.06          & 82.70          & 75.38                 & 83.35        & 85.50 & 84.98    & \underline{85.86}      & \multicolumn{1}{c}{\cellcolor[HTML]{FFCCC9}\textbf{86.54}}    \\ \midrule
                 & Base             & 90.07         & 90.26         & \underline{90.66}  & 90.24           & 88.92          & 90.29          & 90.29                 & 90.57  & 88.88          & 90.52   & \textbf{90.73}        & \multicolumn{1}{c}{\cellcolor[HTML]{FFCCC9}89.21}          \\
\textbf{Food101}          & Novel            & 91.15         & 91.25         & 91.58           & 91.50           & 89.73          & 91.54          & 91.12                 & 91.53        & 90.51          & \underline{91.82}  & \textbf{92.35}  & \multicolumn{1}{c}{\cellcolor[HTML]{FFCCC9}91.75} \\
                 & HM               & 90.61         & 90.75         & 91.12     & 90.87           & 89.32          & 90.91          & 90.70                 & 91.05        & 89.69          & \underline{91.17} & \textbf{91.53} & \multicolumn{1}{c}{\cellcolor[HTML]{FFCCC9}90.46}          \\ \midrule
                 & Base             & 27.63         & 37.82         & 33.53           & 36.85           & 28.15          & 39.22          & 30.69                 & 35.79        & \underline{43.18}    & 38.70     & 39.53     & \multicolumn{1}{c}{\cellcolor[HTML]{FFCCC9}\textbf{47.59}} \\
\textbf{FGVCAircraft}     & Novel            & 35.95         & 33.31         & 32.27           & 33.29           & 34.25          & 34.35          & 36.35                 & 35.19        & 35.39          & \textbf{37.51} & 37.16 & \multicolumn{1}{c}{\cellcolor[HTML]{FFCCC9}\underline{37.48}}    \\
                 & HM               & 31.25         & 35.42         & 32.89           & 34.98           & 30.90          & 36.62          & 33.28                 & 35.49        & \underline{38.90}    & 38.10     & 38.31    & \multicolumn{1}{c}{\cellcolor[HTML]{FFCCC9}\textbf{41.93}} \\ \midrule
                 & Base             & 69.35         & 81.23         & 78.39           & 80.82           & 74.40          & 81.85    & 74.57                 & 79.96        & 81.50          & \underline{81.87} & 81.55 & \multicolumn{1}{c}{\cellcolor[HTML]{FFCCC9}\textbf{82.36}}          \\
\textbf{SUN397}           & Novel            & 75.56         & 75.43         & 77.37           & 74.97           & 77.67          & 76.93          & 77.40                 & 78.11        & 77.75          & \textbf{79.16} & \underline{79.05} & \multicolumn{1}{c}{\cellcolor[HTML]{FFCCC9}78.72}    \\
                 & HM               & 72.32         & 78.22         & 77.88           & 77.79           & 76.00          & 79.31          & 75.96                 & 79.02        & 79.58          & \underline{80.49} & 80.28 & \multicolumn{1}{c}{\cellcolor[HTML]{FFCCC9}\textbf{80.50}}    \\ \midrule
                 & Base             & 53.24         & 78.63         & 67.75           & 78.32           & 56.02          & 80.71          & 57.72                 & 78.20        & \underline{84.03}    & 83.22     & 83.18     & \multicolumn{1}{c}{\cellcolor[HTML]{FFCCC9}\textbf{84.18}} \\
\textbf{DTD}              & Novel            & 60.87         & 50.48         & 55.64           & 56.88           & 61.31          & 55.63          & 60.79                 & 53.90        & \textbf{65.05}    & 62.04   & 62.57       & \multicolumn{1}{c}{\cellcolor[HTML]{FFCCC9}\underline{64.71}}    \\
                 & HM               & 56.80         & 61.49         & 61.10           & 65.90           & 58.55          & 65.86          & 59.22                 & 63.81        & \textbf{73.33} & 71.09     & 71.42    & \multicolumn{1}{c}{\cellcolor[HTML]{FFCCC9}\underline{73.17}}    \\ \midrule
                 & Base             & 57.02         & 88.81         & 78.41           & 85.82           & 53.52          & 92.63          & 62.80                 & 84.25        & \underline{94.22}    & 90.39     & 90.64     & \multicolumn{1}{c}{\cellcolor[HTML]{FFCCC9}\textbf{95.86}} \\
\textbf{EuroSAT}          & Novel            & 64.03         & 59.38         & 66.90           & 58.89           & 65.53          & \underline{70.13}    & 62.76                 & \textbf{70.17}  & 62.94          & 68.44   & 68.73       & \multicolumn{1}{c}{\cellcolor[HTML]{FFCCC9}69.75}          \\
                 & HM               & 60.32         & 71.17         & 72.20           & 69.85           & 58.92          & \underline{79.82}          & 62.78                 & 76.57        & 75.47          & 77.90  & 78.18  & \multicolumn{1}{c}{\cellcolor[HTML]{FFCCC9}\textbf{80.75}} \\ \midrule
                 & Base             & 70.89         & 83.71         & 78.71           & 83.54           & 71.53          & 84.89          & 75.52                 & 81.76        & \underline{86.20}    & 85.52    & 85.47      & \multicolumn{1}{c}{\cellcolor[HTML]{FFCCC9}\textbf{86.49}} \\
\textbf{UCF101}           & Novel            & \underline{78.42}   & 74.17         & 74.56           & 74.36           & 73.12          & 78.06          & 77.54                 & 77.93        & 77.73          & 78.35    & \underline{79.57}      & \multicolumn{1}{c}{\cellcolor[HTML]{FFCCC9}\textbf{79.58}} \\
                 & HM               & 74.47         & 78.65         & 76.58           & 78.68           & 72.32          & 81.33          & 76.52                 & 79.80        & 81.75          & 81.78  & \underline{82.42}  & \multicolumn{1}{c}{\cellcolor[HTML]{FFCCC9}\textbf{82.89}} \\ \midrule
                 & Base             & 69.50         & 82.03         & 77.42           & 81.16           & 70.27          & 83.00          & 72.57                 & 80.20        & \underline{84.15}    & 83.11     & 83.15     & \multicolumn{1}{c}{\cellcolor[HTML]{FFCCC9}\textbf{85.22}} \\
\textbf{Average}          & Novel            & 74.35         & 71.83         & 73.53           & 72.62           & 74.27          & 74.05          & 74.27                 & 74.53        & 74.43          & 75.41  & \underline{75.78}  & \multicolumn{1}{c}{\cellcolor[HTML]{FFCCC9}\textbf{76.25}} \\
                 & HM               & 71.84         & 76.59         & 75.42           & 76.65           & 72.21          & 78.27          & 73.41                 & 77.26        & 78.99          & 79.07  & \underline{79.29}  & \multicolumn{1}{c}{\cellcolor[HTML]{FFCCC9}\textbf{80.49}} \\ \bottomrule[1.5pt]
\end{tabular}%
}
\vspace{-0.3cm}
\caption{Comparison of \emph{A$_3$B$_2$} with 11 methods in base-to-novel generalization across 11 datasets.The averages with 3 different seeds are reported. \textbf{Best results} in bold, \underline{next best} underlined.}
\label{tab:base2new}
\vspace{-0.2cm}
\end{table*}

Similarly, for the image encoder, we introduce an adapter $\mathcal{A}^{v}$ at each layer to modify Eq.~\ref{eq:image_inert}:

\begin{align}
\begin{cases}
[c_i, x_i] = \mathcal{V}_i([c_{i-1}, x_{i-1}]) + \alpha \underline{\mathcal{A}_i^{v}([c_{i-1}, x_{i-1}])}, \\ 
\quad \quad \quad \quad \quad \quad \quad \quad \quad \quad \quad \quad \quad i = 1, 2, \dots, L, \\
\mathcal{A}_i^{v}(z=[c_{i-1}, x_{i-1}]) = \sum_k^n \overset{v}{\mathfrak{w}}_i^k {W_{up}^{v}}_i^k \\
\quad \quad \quad \quad \quad \quad \quad \quad \quad \quad \left(\texttt{ReLU}({W_{down}^{v}}_i(z))\right)
\end{cases}
\end{align}
where ${W_{down}^{v}}_i\in \mathbb{R}^{d_{hv}\times r}$ and ${W_{up}^{v}}_i^k\in \mathbb{R}^{r\times d_{hv}}$ are the ``Down'' expert matrix and $k$-th ``Up'' matrix of image encoder layer $\mathcal{V}_i$, respectively, with $d_{hv}$ as the hidden dimension of the image encoder. $\overset{v}{\mathfrak{w}}_i\in \mathbb{R}^{n}$ represents the gating weights from the router:
\begin{align}
\overset{v}{\mathfrak{w}}_i = \texttt{softmax}({\overset{v}{W}}_i^g z), \label{eq:router_image}
\end{align}
where ${\overset{v}{W}}_i^g\in \mathbb{R}^{d_{hv}\times n}$.

Rather than rigidly locking the image encoder, we fine-tune the image adapters while constraining their optimization via a novel uncertainty-aware mechanism, described below.
\subsubsection{\textit{Optimization Objectives}}
To effectively alleviate \emph{Branch Bias} and ensure structural diversity within the adapters~\cite{mu2025comprehensive}, we introduce two auxiliary objectives: \textit{Uncertainty-Aware Adapter Dampening} (\emph{UAAD}) and \textit{Load Balancing Regularization}.

\paragraph{Uncertainty-Aware Adapter Dampening}
\texttt{Insight III} suggests that aggressive modification of image features on out-of-distribution (OOD) data degrades performance. To address this without manual locking, we propose \emph{UAAD}. This mechanism leverages the model's prediction confidence to gauge whether a sample lies within the distribution. \textbf{Low confidence implies potential OOD data, triggering a penalty on the image adapter's activation magnitude}~\cite{hendrycks2016baseline}.

Let $\mathbf{p}(\hat{y}|x)$ be the softmax probability distribution for the image feature $x$. We define the confidence score as $\kappa = \max \mathbf{p}(\hat{y}|x)$. The bias alleviation loss $\mathcal{L}_{bias}$ is formulated as:
\begin{align}
\mathcal{L}_{bias} = \frac{1}{B} \sum_{j=1}^B (1 - \kappa_j) \cdot \sum_{i=1}^L \| \Delta \mathbf{v}_{i,j} \|_2,
\end{align}
where $B$ is the batch size, and $\Delta \mathbf{v}_{i,j} = \mathcal{A}_i^v(z_j)$ represents the feature modification vector produced by the image adapter at layer $i$. When the model is uncertain (low $\kappa$), $\mathcal{L}_{bias}$ forces the adapter output $\Delta \mathbf{v}$ towards zero, effectively reverting the model to the robust pre-trained CLIP features.

\paragraph{Load Balancing Regularization}

Standard MoE routing often suffer from expert collapse. To prevent this, we enforce a uniform distribution of expert usage across the batch. Unlike methods relying on selection frequency~\cite{fedus2022switch}, we minimize the variance of the routing probabilities directly. Let $\bar{\mathfrak{w}}_{l,j,k}$ (derived from Eq.~\ref{eq:router_text} and Eq.~\ref{eq:router_image}) denote the gating probability of the $k$-th expert for the $j$-th sample in the $l$-th layer, averaged over the batch. We adopt a coefficient of variation~\cite{shazeer2017outrageously} based loss:
\begin{align}
\mathcal{L}_{bal} = n^2 \cdot \frac{1}{L \cdot n} \sum_{l=1}^L \sum_{k=1}^n \left( \left( \frac{1}{B} \sum_{j=1}^B \bar{\mathfrak{w}}_{l,j,k} \right) - \frac{1}{n} \right)^2,
\end{align}
where $1/n$ represents the target uniform probability. Minimizing this objective penalizes deviations from a balanced distribution, ensuring diverse expert utilization.

\paragraph{Total Objective}
The model is primarily optimized using the cross-entropy loss, derived from the classification probabilities defined in Sec.~\ref{sec:preliminary}. For a batch of size $B$, given the ground-truth label $y_j$ for the image future $x^j$, this loss is defined as:
\begin{align}
\mathcal{L}_{CE} = - \frac{1}{B} \sum_{j=1}^B \log \textbf{p}(\hat{y} = y_j \mid x^j).
\end{align}

The final training objective is a weighted sum of the cross-entropy loss and the auxiliary regularizations:
\begin{align}
\mathcal{L}_{total} = \mathcal{L}_{CE} + \lambda_{bias}\mathcal{L}_{bias} + \lambda_{bal}\mathcal{L}_{bal},
\end{align}
where $\lambda_{bias}$ governs the suppression strength for uncertain samples, and $\lambda_{bal}$ enforces the diversity of expert utilization. This formulation allows \emph{A$_3$B$_2$} to adaptively balance between general linguistic knowledge and task-specific visual adaptation, dynamically alleviating branch bias during training.

\begin{table*}[!t]
\centering
\resizebox{0.9\textwidth}{!}{%
\begin{tabular}{@{}c|c|cccccccccc|c@{}}
\toprule[1.5pt]
\midrule
\rotatebox{35}{\textbf{Method}} & \rotatebox{35}{\textbf{ImageNet}} & \rotatebox{35}{\textbf{Caltech101}}           & \rotatebox{35}{\textbf{OxfordPets}}                 & \rotatebox{35}{\textbf{StanfordCars}}                  & \rotatebox{35}{\textbf{Flowers102}}           & \rotatebox{35}{\textbf{Food101}}              & \rotatebox{35}{\textbf{FGVCAircraft}}               & \rotatebox{35}{\textbf{SUN397}}                     & \rotatebox{35}{\textbf{DTD}}                           & \rotatebox{35}{\textbf{EuroSAT}}              & \rotatebox{35}{\textbf{UCF101}}               & \rotatebox{35}{\textbf{Average}}                    \\ \midrule 
\textbf{CoOp}   & 71.30             & 94.20                         & 89.91                               & 64.20                                  & 70.66                         & 85.82                         & 22.34                               & 66.36                               & 43.36                                  & \underline{48.50}                & 68.90                   & 65.43                               \\
\textbf{KgCoOp} & 71.12             & \underline{94.24}                & 90.29                               & \textbf{65.39}                                  & 71.23                         & 86.24                   & 21.99                               & 67.03                               & 44.39                                  & 43.89                         & 68.67                         & 65.34                               \\
\textbf{MaPLe}  & 71.91             & 93.39                         & 90.41                               & 64.37                                  & 70.28                         & 85.63                         & 22.88                               & 66.04                               & 44.60                                  & 39.87                         & 67.91                         & 64.54                               \\
\textbf{TCP}    & 70.36             & 94.21                   & 90.16                               & 65.23                                  & 71.46                   & \underline{86.73}                & 23.74                               & 67.13                               & \underline{45.06}                            & 43.00                         & \underline{68.91}                & 65.56                               \\
\textbf{MMA}    & \underline{72.45}       & 92.39                         & 89.59                               & 61.69                                  & 68.52                         & 81.94                         & 24.19                               & 66.32                               & 44.29                                  & 39.70                         & 67.86                         & 63.65                               \\
\textbf{MMRL}   & 72.10             & 93.90                         & 91.31                      & 65.07                                  & \underline{72.22}                & 85.68                         & \textbf{25.44}                      & \underline{67.24}                      & 44.78                                  & \textbf{48.66}                   & 68.07                         & 66.24                      \\
\textbf{MMRL++}   & 71.96             & 93.97                         & \textbf{91.43}                      & \underline{65.38}                                  & \textbf{72.57}                & 85.72                         & \underline{25.28}                      & \textbf{67.65}                      & 44.83                                  & 48.49                   & 68.52                         & \underline{66.38}                      \\ \midrule
\textbf{\emph{A$_3$B$_2$}}   & \cellcolor[HTML]{FFCCC9}\textbf{72.57}    & \cellcolor[HTML]{FFCCC9}\textbf{94.52} & \cellcolor[HTML]{FFCCC9}\underline{91.36} & \cellcolor[HTML]{FFCCC9}65.32 & \cellcolor[HTML]{FFCCC9}72.15 & \cellcolor[HTML]{FFCCC9}\textbf{86.75} & \cellcolor[HTML]{FFCCC9}24.53 & \cellcolor[HTML]{FFCCC9}67.15 & \cellcolor[HTML]{FFCCC9}\textbf{45.30} & \cellcolor[HTML]{FFCCC9}47.42 & \cellcolor[HTML]{FFCCC9}\textbf{69.43} & \multicolumn{1}{c}{\cellcolor[HTML]{FFCCC9}\textbf{66.39}} \\ \bottomrule[1.5pt]
\end{tabular}%
}
\vspace{-0.3cm}
\caption{Comparison of \emph{A$_3$B$_2$} and 7 leading methods in cross-dataset evaluation, with results of other methods provided in the Appendix~\ref{sec:additional}.}
\label{tab:cross}
\vspace{-0.2cm}
\end{table*}

\section{Experiments}
\label{sec:experiments}

We comprehensively evaluate \emph{A$_3$B$_2$}'s performance around three core image classification tasks in few-shot setting and validate the effectiveness of the discovered insights.

\begin{table}[ht]
\centering
\resizebox{0.85\columnwidth}{!}{%
\begin{tabular}{@{}c|c|cccc|c@{}}
\toprule[1.5pt]
\midrule
\textbf{Method} & \textbf{ImageNet}                   & \textbf{-V2}                        & \textbf{-S}                   & \textbf{-A}                   & \textbf{-R}                            & \textbf{Average}                    \\ \midrule 
\textbf{CoOp}   & \multicolumn{1}{c|}{71.30}          & 64.52                               & \underline{49.02}                & 51.29                   & 76.69                                  & \underline{60.38}                      \\
\textbf{KgCoOp} & \multicolumn{1}{c|}{71.12}          & 64.33                               & 48.97                   & \underline{51.41}                & 76.60                                  & 60.33                               \\
\textbf{MaPLe}  & \multicolumn{1}{c|}{71.91}          & 64.69                               & 48.91                         & 49.26                         & 76.47                                  & 59.83                               \\
\textbf{TCP}    & \multicolumn{1}{c|}{70.36}          & 63.60                               & 48.74                         & 50.50                         & \underline{76.70}                            & 59.89                               \\
\textbf{MMA}    & \multicolumn{1}{c|}{\underline{72.45}}    & \underline{65.23}                      & 48.09                         & 48.07                         & 75.43                                  & 59.21                               \\
\textbf{MMRL}   & \multicolumn{1}{c|}{72.10}          & 64.35                               & 48.70                         & 48.82                         & 76.40                                  & 59.57                               \\
\textbf{MMRL++}   & \multicolumn{1}{c|}{71.97}          & 64.61                               & 48.95                         & 48.69                         & 76.52                                  & 59.69                               \\ \midrule
\textbf{\emph{A$_3$B$_2$}}   & \multicolumn{1}{c|}{\cellcolor[HTML]{FFCCC9}\textbf{72.52}} & \cellcolor[HTML]{FFCCC9}\textbf{65.47} & \cellcolor[HTML]{FFCCC9}\textbf{49.13} & \cellcolor[HTML]{FFCCC9}\textbf{51.75} & \cellcolor[HTML]{FFCCC9}\textbf{76.91} & \multicolumn{1}{c}{\cellcolor[HTML]{FFCCC9}\textbf{60.81}} \\ \bottomrule[1.5pt]
\end{tabular}%
}
\vspace{-0.3cm}
\caption{Comparison of \emph{A$_3$B$_2$} and 7 leading methods in domain generalization, with results of other methods provided in the Appendix~\ref{sec:additional}.}
\label{tab:domain}
\vspace{-0.2cm}
\end{table}

\begin{table*}[ht]
\centering
\small
\setlength{\tabcolsep}{2.5pt}

\begin{tabular}{c|ccc||c|ccc||c|ccc||c|ccc||c|ccc}
\toprule[1.5pt]
\midrule
$n$ & Base & Novel & HM 
& $r$ & Base & Novel & HM
& $\alpha$ & Base & Novel & HM
& $\lambda_{bias}$ & Base & Novel & HM
& $\lambda_{bal}$ & Base & Novel & HM \\
\midrule
1   & 82.75 & 73.95 & 78.10
& 8   & 82.93 & 74.69 & 78.59
& 0.0001 & 85.13 & 75.62 & 80.09
& 0.1 & 84.82 & 75.69 & 80.00
& 0.01 & 85.15 & 75.93 & 80.28 \\

2   & 84.89 & 75.83 & 80.10
& 16  & 84.24 & 75.26 & 79.50
& 0.0005 & 85.07 & 75.98 & 80.27
& 0.2 & 84.93 & 75.91 & 80.17
& 0.05 & \textbf{85.26} & 76.13 & 80.44 \\

\rowcolor[HTML]{EFEFEF}
3   & \textbf{85.22} & \textbf{76.25} & \textbf{80.49}
& 32  & 85.22 & \textbf{76.25} & \textbf{80.49}
& 0.001 & 85.22 & \textbf{76.25} & \textbf{80.49}
& 0.3 & \textbf{85.22} & \textbf{76.25} & \textbf{80.49}
& 0.1 & 85.22 & \textbf{76.25} & \textbf{80.49} \\

4   & 84.92 & 75.52 & 79.94
& 64  & \textbf{85.37} & 75.63 & 80.21
& 0.005 & 85.25 & 76.15 & 80.44
& 0.4 & 85.08 & 75.62 & 80.07
& 0.15 & 85.23 & 76.17 & 80.45 \\

5   & 83.41 & 74.73 & 78.83
& 128 & 84.57 & 73.82 & 78.83
& 0.01 & \textbf{85.38} & 76.09 & 80.47
& 0.5 & 84.67 & 75.59 & 79.87
& 0.2 & 85.03 & 75.82 & 80.16 \\
\bottomrule[1.5pt]
\end{tabular}
\vspace{-0.3cm}
\caption{Analysis experiments on the \emph{A$_3$B$_2$} hyperparameters.}
\label{tab:parameter}
\vspace{-0.2cm}
\end{table*}

\subsection{Experimental Setup}

\paragraph{Base-to-Novel Generalization} Consistent with previous studies~\cite{brown2020language,guo2025mmrl}, this evaluation divides the dataset into base and novel categories. The model is trained on base categories under a few-shot setting and tested separately on both base and novel categories. This evaluation includes 11 visual classification benchmark datasets: ImageNet~\cite{deng2009imagenet}, Caltech101~\cite{fei2004learning}, OxfordPets~\cite{parkhi2012cats}, StanfordCars~\cite{krause20133d}, Flowers102~\cite{nilsback2008automated}, Food101~\cite{bossard2014food}, FGVCAircraft~\cite{maji2013fine}, SUN397~\cite{xiao2010sun}, DTD~\cite{cimpoi2014describing}, EuroSAT~\cite{helber2019eurosat}, and UCF101~\cite{soomro2012ucf101}.

\paragraph{Cross-Dataset Evaluation} This evaluation measures a model's discriminative ability on unseen datasets and its grasp of general knowledge. Following CoCoOp~\cite{zhou2022conditional}, the model is trained on ImageNet with 1k classes and 16 shots per class, then directly evaluated on 10 other datasets mentioned in the base-to-novel generalization task.

\paragraph{Domain Generalization} This task evaluates the model's robustness to domain shifts and its discriminative ability to out-of-distribution data. Following the setup of CoCoOp~\cite{zhou2022conditional}, we fine-tune the model on ImageNet and test it on four of its variants: ImageNetV2~\cite{recht2019imagenet}, ImageNet-Sketch~\cite{wang2019learning}, ImageNet-A~\cite{hendrycks2021natural}, and ImageNet-R~\cite{hendrycks2021many}.

\paragraph{Baselines}
We compare the proposed \emph{A$_3$B$_2$} with 11 competitive methods, including: CLIP (zero-shot)~\cite{radford2021learning}, CoOp~\cite{radford2021learning}, CoCoOp~\cite{zhou2022conditional}, KgCoOp~\cite{yao2023visual}, RPO~\cite{lee2023read}, MaPLe~\cite{khattak2023maple}, CLIP-Adapter~\cite{gao2024clip}, TCP~\cite{yao2024tcp}, MMA~\cite{yang2024mma}, MMRL~\cite{guo2025mmrl}, and
MMRL++~\cite{guo2025mmrl++}.

\paragraph{Implementation Details}

Consistent with previous studies~\cite{yang2024mma,guo2025mmrl++}, all experiments are conducted using a pre-trained transformer-based ViT-B/16 backbone of the CLIP model under a 16-shot setting (\textit{i.e.}, 16 training samples per class). To ensure consistency with prior work~\cite{radford2021learning,zhou2022conditional,yang2024mma} as much as possible, we set the batch size to 8, with 5 epochs for ImageNet and 10 epochs for other datasets. The same image augmentation methods are used, and the SGD optimizer is employed with an initial learning rate of 0.0015, a momentum of 0.9, and a weight decay of 0.0005, combined with a cosine annealing scheduler with 1 epoch of warm-up. The text templates for the proposed \emph{A$_3$B$_2$} follow~\cite{gao2024clip}, with the default number of experts $n=3$, low-rank $r=32$, coefficient $\alpha=0.001$, suppression strength $\lambda_{bias}=0.3$ and balance parameter $\lambda_{bal}=0.1$. The optimization strategies for other baselines follow the settings of the original papers. \textbf{For fairness, all baselines are re-implemented under the same experimental setup} and trained with mixed-precision acceleration on a single NVIDIA A800 GPU. Final performance is reported as the average over three random seeds (seed=1,2,3). Our code and datasets are available at \url{https://github.com/zyy-2001/A3B2}.

\subsection{Base-to-Novel Generalization}

In this task, we report the model's accuracy on both base and novel classes, along with their harmonic mean (HM), as shown in Table~\ref{tab:base2new}. Overall, our method consistently outperforms all baselines across all metrics and datasets. The key findings are as follows:

\begin{itemize}[leftmargin=*]
    \item{Among all baselines, the recently proposed MMA and MMRL++ achieve the best average performance. MMRL++ leverages its decoupled inference strategy and progressive inter-layer composition to achieve leading novel-class accuracy with high parameter efficiency. MMA, on the other hand, integrates features from different branches into a shared space, enhancing gradient flow across branches, thus improving consistency between modalities and achieving optimal base-class recognition.}
    
    \item{The proposed \emph{A$_3$B$_2$} consistently leads all metrics across 11 datasets. Compared to MMA, which performs best on base-class recognition, and MMRL++, which excels on the novel and HM metrics, \textbf{\emph{A$_3$B$_2$} shows a significant advantage in average performance}. These impressive results validate the effectiveness of our insights.}
    
    \item{In this in-distribution task, methods with dual-branch optimization (\textit{e.g.}, MaPLe, MMA, and MMRL++) generally outperform single-branch optimization methods (\textit{e.g.}, CoOp, CoCoOp, KgCoOp, and TCP), \textbf{further confirming the universality and broad applicability of our \texttt{Insight II}}.}
\end{itemize}

\subsection{Cross-Dataset Evaluation}
We have compared the top 7 methods in the base-to-novel generalization task with the proposed \emph{A$_3$B$_2$} in the cross-dataset evaluation task, as shown in Table~\ref{tab:cross}. We observe that MMA performs poorly in our experimental setup, which we attribute to MMA simultaneously adapting the image encoder in out-of-distribution tasks, contrary to our \texttt{Insight III}. And this experiment validates the broad applicability of the \texttt{Insight III}: Text-only encoder optimization methods (\textit{e.g.}, CoOp, CoCoOp, KgCoOp) generally outperform methods with image encoder optimization (\textit{e.g.}, RPO, MaPLe, MMA). Notably, the MMRL and MMRL++ with image encoder optimization even outperform the proposed \emph{A$_3$B$_2$} in some cases. We believe these unexpected results stem from their use of a loss-function–driven model exploration strategy during optimization, differing from other methods, whereas we adopt a more general and adaptive structural design based on exploratory observations.

\subsection{Domain Generalization}
As shown in Table~\ref{tab:domain}, the text-only encoder optimization
methods significantly outperform the methods with image encoder optimization, \textbf{further validating the effectiveness of the \texttt{Insight III}}. Meanwhile, we also observe that the MMRL and MMRL++ methods perform poorly on this task, while our \emph{A$_3$B$_2$} method consistently achieves leading performance. This suggests that the \emph{A$_3$B$_2$}, which focuses more on model structure optimization, is more stable, providing strong support for its future application and the insights derived for related research.

\subsection{In-depth Analysis}

We conduct analytical experiments on 11 datasets for the base-to-novel task; detailed results are in the Appendix~\ref{sec:additional}.
\paragraph{Few-Shot Learning}

\begin{figure}[t]
\centering
    \vspace{-0.3cm}
    \includegraphics[width=0.95\linewidth]{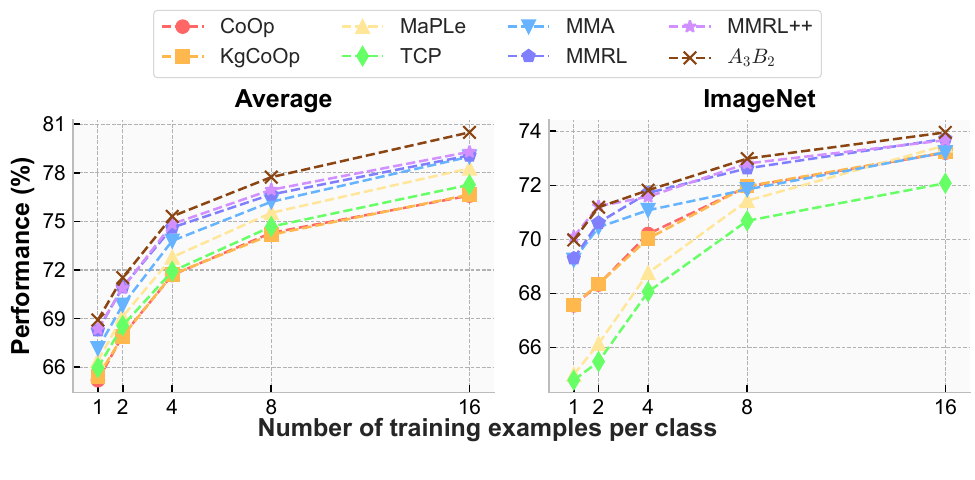}
    \vspace{-0.6cm}
    \textbf{\caption{Comparison (HM) of \emph{A$_3$B$_2$} and 7 leading methods on few-shot learning, with results on all datasets provided in the Appendix~\ref{sec:additional}.}
    \label{fig:shot}}
    \vspace{-0.4cm}
\end{figure}
As shown in Fig.~\ref{fig:shot}, \emph{A$_3$B$_2$} consistently achieves the best performance across 11 datasets under all shot settings, with its advantage increasing as the number of shots grows. This demonstrates its strong transfer capability and robustness, even in data-scarce few-shot scenarios.

\paragraph{Sensitivity Analysis}
We analyze the sensitivity of the \emph{A$_3$B$_2$} hyperparameters, as shown in Table \ref{tab:parameter}.

\textbf{Number of Experts $n$.} Too few experts limit model expressiveness, while too many may lead to overfitting. The optimal balance is achieved when $n=3$.

\textbf{Low-rank $r$.} The low-rank parameter faces a similar trade-off to the number of experts. Performance is best when $r=32$.

\textbf{Balance Coefficient $\alpha$.} The coefficient $\alpha$ controls the relationship between pre-trained and learnable weights, which are more effective for base and novel settings, respectively. A good balance is achieved when $\alpha=0.001$.

\textbf{Suppression Strength $\lambda_{bias}$.} This term controls the strength of uncertainty-aware dampening on image adapters. Moderate suppression ($\lambda_{bias}=0.3$) best balances robustness and adaptability, while overly large values over-constrain visual adaptation.

\textbf{Balance Parameter $\lambda_{bal}$.} This parameter regulates expert utilization diversity. A moderate value ($\lambda_{bal}=0.1$) effectively prevents expert collapse and yields the most stable performance across base and novel classes.

\paragraph{Ablation Study} As shown in Fig.~\ref{fig:ablation}, removing any component degrades performance, confirming their effectiveness. Notably, removing $\mathcal{L}_{bias}+\mathcal{L}_{bal}$ leads to the largest drop, highlighting their synergy. Applying $\mathcal{L}_{bias}$ to the text encoder (``w. bi-encoder $\mathcal{L}_{bias}$'') also causes a significant decline, validating the applicability of \texttt{Insight II \& III}.

We provide a computational cost analysis of \emph{A$_3$B$_2$} and additional experiments on ResNet-based CLIP in Appendix~\ref{sec:cost} and~\ref{sec:cnn}, respectively.

\begin{figure}[t]
\centering
    \vspace{-0.2cm}\includegraphics[width=\linewidth]{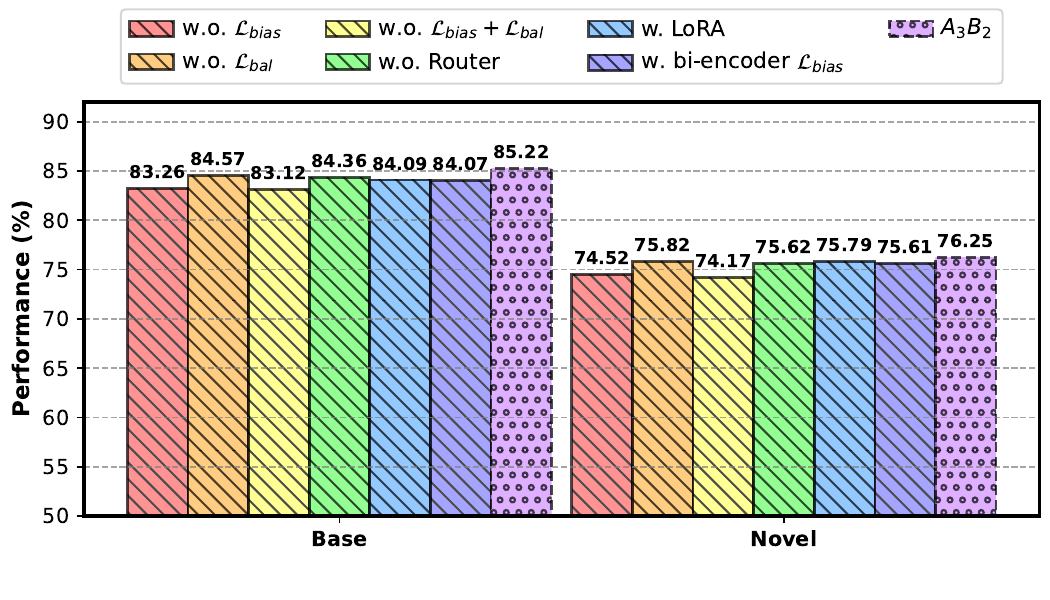}
    \vspace{-1.1cm}
    \textbf{\caption{The performance of different variants in \emph{A$_3$B$_2$}.}
    \label{fig:ablation}}
    \vspace{-0.4cm}
\end{figure}

\section{Conclusion}
\label{sec:conclusion}
This paper identifies the Branch Bias issue in few-shot vision-language image classification, showing that fine-tuning the image encoder can be detrimental to robustness under out-of-distribution settings. Towards this end, we propose \emph{A$_3$B$_2$}, an adaptive asymmetric adapter that modulates the image branch based on prediction uncertainty. Extensive experiments show that \emph{A$_3$B$_2$} consistently improves performance across various few-shot settings, offering an effective solution for branch-aware adaptation in vision-language models.

\section*{Acknowledgments}
This research was supported by grants from the ``Pioneer'' and ``Leading Goose'' R\&D Program of Zhejiang under Grant No. 2025C02022 and the National Natural Science Foundation of China (No.62307032).

\bibliographystyle{named}
\bibliography{ijcai26}

\appendix

We provide supplementary information here, covering all the details mentioned in the main text, including:

\begin{tcolorbox}[title=\textbf{Content}]
\begin{itemize}[leftmargin=*]
\item \textbf{\ref{sec:fixed}~Fixed Asymmetric Structure}
\begin{itemize}
    \item \ref{sec:empirical}~Empirical Evidence
    \item \ref{sec:theoretical}~Theoretical Support
\end{itemize}
\item {\textbf{\ref{sec:guarantee}~Theoretical Guarantees of the Proposed Method}}
\item {\textbf{\ref{sec:motivation}~Detailed Motivation Experiment}}
\item \textbf{\ref{sec:additional}~Additional Experimental Results} 
\item \textbf{\ref{sec:cost}~Computational Overhead} 
\item \textbf{\ref{sec:cnn}~CNN-based CLIP Model Experiment}

\end{itemize}
\end{tcolorbox}

\begin{table*}[t]
\resizebox{\textwidth}{!}{%
\begin{tabular}{@{}c|c|cccccccccccc@{}}
\toprule[1.5pt]
\midrule
\textbf{Setup}                     & \textbf{Method} & \textbf{ImageNet}             & \textbf{Caltech101}           & \textbf{OxfordPets}           & \textbf{StanfordCars}         & \textbf{Flowers102}           & \textbf{Food101}              & \textbf{FGVCAircraft}         & \textbf{SUN397}               & \textbf{DTD}                  & \textbf{EuroSAT}              & \textbf{UCF101}               & \textbf{Average}              \\ \midrule
                                   & \textbf{CoOp}   & 70.29                         & 94.28                         & 91.72                         & 59.81                         & 79.55                         & 87.38                         & 18.63                         & 70.62                         & 57.55                         & 58.43                         & 72.49                         & 69.16                         \\
                                   & \textbf{KgCoOp} & 70.20                         & 94.08                         & 91.32                         & 58.41                         & 79.83                         & 87.58                         & 18.16                         & 70.26                         & 57.31                         & 58.48                         & 72.36                         & 68.91                         \\
                                   & \textbf{MaPLe}  & 69.51                         & 94.72                         & 91.40                         & 60.72                         & 83.40                         & 85.18                         & 21.66                         & 70.17                         & 58.97                         & 58.05                         & 71.70                         & 69.59                         \\
                                   & \textbf{TCP}    & 69.39                         & 94.07                         & 90.96                         & 57.17                         & 80.24                         & 85.48                         & 19.73                         & 68.56                         & 57.13                         & 57.52                         & 70.05                         & 68.21                         \\
                                   & \textbf{MMA}    & 73.01                         & 95.03                         & 93.15                         & 62.53                         & 83.45                         & 87.27                         & 23.56                         & 71.90                         & 59.77                         & 59.18                         & 74.09                         & 71.18                         \\
                                   & \textbf{MMRL}   & 72.41                         & 95.22                         & 92.40                         & 60.08                         & 85.43                         & 88.25                         & 19.19                         & 72.57                         & 59.31                         & 59.60                         & 74.19                         & 70.79                         \\
                                   & \textbf{MMRL++} & 73.72                         & 95.55                         & 92.81                         & 59.84                         & 83.71                         & 87.71                         & 19.09                         & 72.38                         & 59.76                         & 59.63                         & 73.79                         & 70.73                         \\ \cmidrule(l){2-14} 
\multirow{-8}{*}{\textbf{1-shot}}  & \textbf{\emph{A$_3$B$_2$}}   & \cellcolor[HTML]{FFCCC9}73.24 & \cellcolor[HTML]{FFCCC9}95.62 & \cellcolor[HTML]{FFCCC9}93.41 & \cellcolor[HTML]{FFCCC9}62.91 & \cellcolor[HTML]{FFCCC9}83.71 & \cellcolor[HTML]{FFCCC9}87.59 & \cellcolor[HTML]{FFCCC9}23.95 & \cellcolor[HTML]{FFCCC9}72.63 & \cellcolor[HTML]{FFCCC9}59.93 & \cellcolor[HTML]{FFCCC9}59.13 & \cellcolor[HTML]{FFCCC9}74.27 & \multicolumn{1}{c}{\cellcolor[HTML]{FFCCC9}71.49} \\ \midrule
                                   & \textbf{CoOp}   & 71.11                         & 95.77                         & 93.12                         & 63.58                         & 87.35                         & 86.45                         & 22.85                         & 72.37                         & 60.34                         & 60.13                         & 74.76                         & 71.62                         \\
                                   & \textbf{KgCoOp} & 71.03                         & 95.58                         & 92.69                         & 61.04                         & 86.50                         & 86.55                         & 22.25                         & 72.00                         & 60.13                         & 60.82                         & 74.62                         & 71.20                         \\
                                   & \textbf{MaPLe}  & 69.04                         & 96.17                         & 93.19                         & 65.28                         & 89.00                         & 86.18                         & 25.02                         & 72.72                         & 62.79                         & 60.58                         & 74.45                         & 72.22                         \\
                                   & \textbf{TCP}    & 67.84                         & 95.55                         & 92.81                         & 60.44                         & 85.60                         & 86.48                         & 22.80                         & 71.04                         & 60.89                         & 60.51                         & 71.67                         & 70.51                         \\
                                   & \textbf{MMA}    & 74.28                         & 96.17                         & 93.90                         & 66.12                         & 90.12                         & 86.69                         & 26.14                         & 73.33                         & 65.76                         & 61.00                         & 76.09                         & 73.60                         \\
                                   & \textbf{MMRL}   & 73.76                         & 95.89                         & 93.11                         & 63.77                         & 90.66                         & 86.96                         & 22.99                         & 75.39                         & 67.85                         & 61.20                         & 76.61                         & 73.47                         \\
                                   & \textbf{MMRL++} & 74.74                         & 96.42                         & 93.01                         & 63.75                         & 89.19                         & 86.63                         & 21.99                         & 74.69                         & 67.99                         & 61.70                         & 77.08                         & 73.38                         \\ \cmidrule(l){2-14} 
\multirow{-8}{*}{\textbf{2-shot}}  & \textbf{\emph{A$_3$B$_2$}}   & \cellcolor[HTML]{FFCCC9}74.52 & \cellcolor[HTML]{FFCCC9}96.29 & \cellcolor[HTML]{FFCCC9}94.24 & \cellcolor[HTML]{FFCCC9}67.66 & \cellcolor[HTML]{FFCCC9}90.46 & \cellcolor[HTML]{FFCCC9}86.97 & \cellcolor[HTML]{FFCCC9}28.78 & \cellcolor[HTML]{FFCCC9}74.12 & \cellcolor[HTML]{FFCCC9}66.19 & \cellcolor[HTML]{FFCCC9}62.07 & \cellcolor[HTML]{FFCCC9}76.73 & \multicolumn{1}{c}{\cellcolor[HTML]{FFCCC9}74.37} \\ \midrule
                                   & \textbf{CoOp}   & 72.84                         & 96.90                         & 95.97                         & 67.16                         & 92.19                         & 87.55                         & 26.89                         & 76.13                         & 66.08                         & 77.04                         & 78.50                         & 76.11                         \\
                                   & \textbf{KgCoOp} & 72.74                         & 96.84                         & 95.52                         & 64.52                         & 91.31                         & 87.53                         & 26.18                         & 75.75                         & 65.84                         & 74.54                         & 78.36                         & 75.38                         \\
                                   & \textbf{MaPLe}  & 71.80                         & 96.64                         & 94.24                         & 68.65                         & 92.75                         & 86.55                         & 28.26                         & 76.61                         & 69.04                         & 84.75                         & 78.34                         & 77.06                         \\
                                   & \textbf{TCP}    & 70.56                         & 96.02                         & 93.87                         & 63.50                         & 89.20                         & 86.80                         & 25.79                         & 74.89                         & 66.87                         & 78.05                         & 75.46                         & 74.64                         \\
                                   & \textbf{MMA}    & 74.96                         & 96.78                         & 94.15                         & 70.45                         & 92.81                         & 87.30                         & 30.81                         & 76.41                         & 74.13                         & 84.97                         & 80.00                         & 78.43                         \\
                                   & \textbf{MMRL}   & 74.92                         & 97.07                         & 94.14                         & 68.35                         & 94.04                         & 87.23                         & 27.18                         & 77.87                         & 75.00                         & 84.84                         & 80.72                         & 78.31                         \\
                                   & \textbf{MMRL++} & 75.26                         & 97.32                         & 94.28                         & 68.26                         & 93.13                         & 87.43                         & 27.77                         & 77.26                         & 74.95                         & 85.47                         & 80.65                         & 78.34                         \\ \cmidrule(l){2-14} 
\multirow{-8}{*}{\textbf{4-shot}}  & \textbf{\emph{A$_3$B$_2$}}   & \cellcolor[HTML]{FFCCC9}75.20 & \cellcolor[HTML]{FFCCC9}96.85 & \cellcolor[HTML]{FFCCC9}94.47 & \cellcolor[HTML]{FFCCC9}72.19 & \cellcolor[HTML]{FFCCC9}93.21 & \cellcolor[HTML]{FFCCC9}87.66 & \cellcolor[HTML]{FFCCC9}33.96 & \cellcolor[HTML]{FFCCC9}77.82 & \cellcolor[HTML]{FFCCC9}74.19 & \cellcolor[HTML]{FFCCC9}85.42 & \cellcolor[HTML]{FFCCC9}80.34 & \multicolumn{1}{c}{\cellcolor[HTML]{FFCCC9}79.21} \\ \midrule
                                   & \textbf{CoOp}   & 74.85                         & 97.13                         & 94.65                         & 71.52                         & 94.99                         & 88.31                         & 33.99                         & 77.84                         & 72.89                         & 81.65                         & 81.62                         & 79.04                         \\
                                   & \textbf{KgCoOp} & 74.77                         & 96.93                         & 94.19                         & 68.68                         & 94.09                         & 88.27                         & 33.15                         & 77.42                         & 72.61                         & 79.93                         & 81.48                         & 78.32                         \\
                                   & \textbf{MaPLe}  & 74.59                         & 97.43                         & 94.96                         & 72.41                         & 95.88                         & 88.43                         & 34.03                         & 79.37                         & 75.27                         & 87.03                         & 81.22                         & 80.06                         \\
                                   & \textbf{TCP}    & 73.27                         & 96.78                         & 94.58                         & 66.96                         & 92.25                         & 87.75                         & 32.03                         & 77.50                         & 72.92                         & 80.11                         & 78.21                         & 77.49                         \\
                                   & \textbf{MMA}    & 75.75                         & 97.57                         & 94.74                         & 74.98                         & 95.80                         & 88.21                         & 36.71                         & 79.95                         & 77.75                         & 88.20                         & 83.35                         & 81.18                         \\
                                   & \textbf{MMRL}   & 75.86                         & 97.35                         & 94.63                         & 72.17                         & 96.03                         & 88.82                         & 33.29                         & 80.09                         & 79.15                         & 85.87                         & 82.63                         & 80.54                         \\
                                   & \textbf{MMRL++} & 75.97                         & 97.65                         & 94.52                         & 72.28                         & 95.84                         & 88.21                         & 34.46                         & 79.49                         & 79.20                         & 86.46                         & 82.64                         & 80.61                         \\ \cmidrule(l){2-14} 
\multirow{-8}{*}{\textbf{8-shot}}  & \textbf{\emph{A$_3$B$_2$}}   & \cellcolor[HTML]{FFCCC9}76.29 & \cellcolor[HTML]{FFCCC9}97.75 & \cellcolor[HTML]{FFCCC9}95.05 & \cellcolor[HTML]{FFCCC9}76.82 & \cellcolor[HTML]{FFCCC9}96.11 & \cellcolor[HTML]{FFCCC9}88.53 & \cellcolor[HTML]{FFCCC9}38.51 & \cellcolor[HTML]{FFCCC9}80.80 & \cellcolor[HTML]{FFCCC9}78.84 & \cellcolor[HTML]{FFCCC9}89.67 & \cellcolor[HTML]{FFCCC9}83.63 & \multicolumn{1}{c}{\cellcolor[HTML]{FFCCC9}82.00} \\ \midrule
                                   & \textbf{CoOp}   & 76.18                         & 98.36                         & 95.27                         & 74.89                         & 97.12                         & 90.26                         & 37.82                         & 81.23                         & 78.63                         & 88.81                         & 83.71                         & 82.03                         \\
                                   & \textbf{KgCoOp} & 76.08                         & 76.08                         & 94.84                         & 71.93                         & 96.20                         & 90.24                         & 36.85                         & 80.82                         & 78.32                         & 85.82                         & 83.54                         & 81.16                         \\
                                   & \textbf{MaPLe}  & 76.73                         & 98.26                         & 95.21                         & 76.17                         & 97.09                         & 90.29                         & 39.22                         & 81.85                         & 80.71                         & 92.63                         & 84.89                         & 83.00                         \\
                                   & \textbf{TCP}    & 75.39                         & 97.61                         & 94.79                         & 70.47                         & 93.41                         & 90.57                         & 35.79                         & 79.96                         & 78.20                         & 84.25                         & 81.76                         & 80.20                         \\
                                   & \textbf{MMA}    & 77.19                         & 98.54                         & 95.20                         & 78.95                         & 97.79                         & 88.88                         & 43.18                         & 81.50                         & 84.03                         & 94.22                         & 86.20                         & 84.15                         \\
                                   & \textbf{MMRL}   & 77.00                         & 98.19                         & 95.41                         & 75.55                         & 97.79                         & 90.52                         & 38.70                         & 81.87                         & 83.22                         & 90.39                         & 85.52                         & 83.11                         \\
                                   & \textbf{MMRL++} & 76.87                         & 98.26                         & 95.16                         & 75.60                         & 97.68                         & 90.73                         & 39.53                         & 81.55                         & 83.18                         & 90.64                         & 85.47                         & 83.15                         \\ \cmidrule(l){2-14} 
\multirow{-8}{*}{\textbf{16-shot}} & \textbf{\emph{A$_3$B$_2$}}   & \cellcolor[HTML]{FFCCC9}77.37 & \cellcolor[HTML]{FFCCC9}98.73 & \cellcolor[HTML]{FFCCC9}95.52 & \cellcolor[HTML]{FFCCC9}81.96 & \cellcolor[HTML]{FFCCC9}98.14 & \cellcolor[HTML]{FFCCC9}89.21 & \cellcolor[HTML]{FFCCC9}47.59 & \cellcolor[HTML]{FFCCC9}82.36 & \cellcolor[HTML]{FFCCC9}84.18 & \cellcolor[HTML]{FFCCC9}95.86 & \cellcolor[HTML]{FFCCC9}86.49 & \multicolumn{1}{c}{\cellcolor[HTML]{FFCCC9}85.22} \\ \bottomrule[1.5pt]
\end{tabular}%
}
\vspace{-0.3cm}
\caption{Comparison (Base) of \emph{A$_3$B$_2$} with 7 leading methods on few-shot learning across 11 datasets.}
\label{tab:shot-base}
\end{table*}

\begin{table*}[t]
\centering
\resizebox{\textwidth}{!}{%
\begin{tabular}{c|c|cccccccccccc}
\toprule[1.5pt]
\midrule
\textbf{Setup}                     & \textbf{Method} & \textbf{ImageNet}             & \textbf{Caltech101}           & \textbf{OxfordPets}           & \textbf{StanfordCars}         & \textbf{Flowers102}           & \textbf{Food101}              & \textbf{FGVCAircraft}         & \textbf{SUN397}               & \textbf{DTD}                  & \textbf{EuroSAT}              & \textbf{UCF101}               & \textbf{Average}              \\ \midrule
                                   & \textbf{CoOp}   & 65.03                         & 91.26                         & 95.20                         & 58.21                         & 57.47                         & 88.40                         & 16.40                         & 67.42                         & 36.29                         & 38.42                         & 64.23                         & 61.67                         \\
                                   & \textbf{KgCoOp} & 65.14                         & 91.30                         & 95.35                         & 59.33                         & 58.25                         & 88.63                         & 16.41                         & 67.02                         & 40.91                         & 38.12                         & 64.41                         & 62.26                         \\
                                   & \textbf{MaPLe}  & 61.04                         & 90.79                         & 94.00                         & 58.44                         & 61.87                         & 86.37                         & 18.95                         & 65.94                         & 40.66                         & 54.51                         & 65.96                         & 63.50                         \\
                                   & \textbf{TCP}    & 60.77                         & 91.13                         & 93.65                         & 60.07                         & 64.61                         & 86.35                         & 19.40                         & 66.92                         & 39.40                         & 54.49                         & 65.85                         & 63.88                         \\
                                   & \textbf{MMA}    & 65.84                         & 90.27                         & 95.52                         & 57.42                         & 64.80                         & 88.86                         & 19.29                         & 66.69                         & 46.26                         & 37.52                         & 66.79                         & 63.57                         \\
                                   & \textbf{MMRL}   & 66.45                         & 91.19                         & 94.18                         & 59.59                         & 65.65                         & 89.52                         & 18.57                         & 70.22                         & 46.44                         & 55.69                         & 67.97                         & 65.95                         \\
                                   & \textbf{MMRL++} & 66.82                         & 91.75                         & 95.14                         & 59.19                         & 65.60                         & 89.31                         & 17.91                         & 70.38                         & 47.17                         & 56.15                         & 68.21                         & 66.15                         \\\cmidrule(l){2-14}
\multirow{-8}{*}{\textbf{1-shot}}  & \textbf{\emph{A$_3$B$_2$}}   & \cellcolor[HTML]{FFCCC9}66.99 & \cellcolor[HTML]{FFCCC9}91.81 & \cellcolor[HTML]{FFCCC9}95.73 & \cellcolor[HTML]{FFCCC9}59.83 & \cellcolor[HTML]{FFCCC9}66.02 & \cellcolor[HTML]{FFCCC9}90.06 & \cellcolor[HTML]{FFCCC9}19.39 & \cellcolor[HTML]{FFCCC9}70.48 & \cellcolor[HTML]{FFCCC9}47.05 & \cellcolor[HTML]{FFCCC9}55.89 & \cellcolor[HTML]{FFCCC9}68.43 & \cellcolor[HTML]{FFCCC9}66.52 \\ \midrule
                                   & \textbf{CoOp}   & 65.76                         & 91.74                         & 95.55                         & 60.88                         & 64.70                         & 88.48                         & 20.09                         & 67.20                         & 42.75                         & 47.57                         & 66.23                         & 64.63                         \\
                                   & \textbf{KgCoOp} & 65.87                         & 92.17                         & 95.20                         & 62.30                         & 65.56                         & 88.70                         & 20.11                         & 66.82                         & 43.67                         & 47.20                         & 66.38                         & 64.91                         \\
                                   & \textbf{MaPLe}  & 63.48                         & 92.14                         & 95.89                         & 62.80                         & 66.03                         & 87.42                         & 21.92                         & 68.33                         & 43.31                         & 59.46                         & 68.48                         & 66.30                         \\
                                   & \textbf{TCP}    & 63.28                         & 92.51                         & 95.54                         & 64.61                         & 68.99                         & 87.41                         & 22.41                         & 69.41                         & 41.99                         & 59.53                         & 68.38                         & 66.73                         \\
                                   & \textbf{MMA}    & 66.99                         & 91.37                         & 96.28                         & 60.77                         & 70.04                         & 88.28                         & 21.43                         & 70.00                         & 50.35                         & 46.78                         & 68.60                         & 66.44                         \\
                                   & \textbf{MMRL}   & 67.74                         & 91.83                         & 94.90                         & 63.27                         & 69.65                         & 90.23                         & 22.27                         & 72.88                         & 50.58                         & 60.73                         & 70.19                         & 68.57                         \\
                                   & \textbf{MMRL++} & 68.04                         & 92.58                         & 95.37                         & 63.12                         & 69.96                         & 90.21                         & 20.70                         & 72.39                         & 51.17                         & 59.69                         & 70.75                         & 68.54                         \\\cmidrule(l){2-14}
\multirow{-8}{*}{\textbf{2-shot}}  & \textbf{\emph{A$_3$B$_2$}}   & \cellcolor[HTML]{FFCCC9}68.15 & \cellcolor[HTML]{FFCCC9}92.90 & \cellcolor[HTML]{FFCCC9}96.56 & \cellcolor[HTML]{FFCCC9}63.36 & \cellcolor[HTML]{FFCCC9}70.36 & \cellcolor[HTML]{FFCCC9}89.45 & \cellcolor[HTML]{FFCCC9}22.64 & \cellcolor[HTML]{FFCCC9}71.90 & \cellcolor[HTML]{FFCCC9}51.03 & \cellcolor[HTML]{FFCCC9}60.18 & \cellcolor[HTML]{FFCCC9}71.24 & \cellcolor[HTML]{FFCCC9}68.89 \\ \midrule
                                   & \textbf{CoOp}   & 67.75                         & 92.02                         & 96.56                         & 64.30                         & 68.28                         & 88.52                         & 23.68                         & 70.68                         & 42.39                         & 60.49                         & 70.56                         & 67.75                         \\
                                   & \textbf{KgCoOp} & 67.52                         & 92.46                         & 96.23                         & 65.76                         & 69.19                         & 88.76                         & 23.64                         & 70.25                         & 47.76                         & 60.10                         & 70.72                         & 68.40                         \\
                                   & \textbf{MaPLe}  & 65.97                         & 92.58                         & 96.95                         & 66.10                         & 68.83                         & 87.70                         & 24.73                         & 72.00                         & 47.57                         & 64.16                         & 72.02                         & 68.96                         \\
                                   & \textbf{TCP}    & 65.73                         & 93.00                         & 96.60                         & 67.97                         & 71.88                         & 87.65                         & 25.29                         & 73.07                         & 46.07                         & 64.18                         & 71.89                         & 69.39                         \\
                                   & \textbf{MMA}    & 67.58                         & 91.70                         & 96.56                         & 64.72                         & 72.10                         & 87.89                         & 25.25                         & 72.87                         & 56.59                         & 59.13                         & 72.13                         & 69.68                         \\
                                   & \textbf{MMRL}   & 68.79                         & 92.98                         & 95.93                         & 67.81                         & 72.25                         & 90.49                         & 26.33                         & 75.32                         & 55.92                         & 64.25                         & 73.92                         & 71.27                         \\
                                   & \textbf{MMRL++} & 68.25                         & 93.44                         & 96.69                         & 67.56                         & 73.00                         & 91.02                         & 26.09                         & 74.87                         & 56.39                         & 65.61                         & 75.11                         & 71.64                         \\\cmidrule(l){2-14}
\multirow{-8}{*}{\textbf{4-shot}}  & \textbf{\emph{A$_3$B$_2$}}   & \cellcolor[HTML]{FFCCC9}68.73 & \cellcolor[HTML]{FFCCC9}93.24 & \cellcolor[HTML]{FFCCC9}96.74 & \cellcolor[HTML]{FFCCC9}67.86 & \cellcolor[HTML]{FFCCC9}73.45 & \cellcolor[HTML]{FFCCC9}91.07 & \cellcolor[HTML]{FFCCC9}26.73 & \cellcolor[HTML]{FFCCC9}74.80 & \cellcolor[HTML]{FFCCC9}56.27 & \cellcolor[HTML]{FFCCC9}65.93 & \cellcolor[HTML]{FFCCC9}74.80 & \cellcolor[HTML]{FFCCC9}71.78 \\ \midrule
                                   & \textbf{CoOp}   & 69.27                         & 93.03                         & 97.13                         & 68.42                         & 70.39                         & 89.60                         & 29.92                         & 72.25                         & 50.79                         & 57.58                         & 72.36                         & 70.07                         \\
                                   & \textbf{KgCoOp} & 69.38                         & 93.50                         & 96.76                         & 70.02                         & 71.29                         & 89.87                         & 29.91                         & 71.82                         & 52.72                         & 57.10                         & 72.55                         & 70.45                         \\
                                   & \textbf{MaPLe}  & 68.52                         & 93.35                         & 97.35                         & 69.73                         & 71.16                         & 89.67                         & 29.82                         & 74.57                         & 51.84                         & 65.63                         & 74.69                         & 71.48                         \\
                                   & \textbf{TCP}    & 68.29                         & 93.72                         & 97.32                         & 71.72                         & 74.34                         & 89.62                         & 30.56                         & 75.67                         & 51.19                         & 65.66                         & 74.54                         & 72.06                         \\
                                   & \textbf{MMA}    & 68.34                         & 92.70                         & 97.12                         & 68.91                         & 74.39                         & 88.84                         & 30.10                         & 75.32                         & 60.19                         & 58.91                         & 75.15                         & 71.82                         \\
                                   & \textbf{MMRL}   & 69.67                         & 93.26                         & 96.44                         & 71.65                         & 73.77                         & 91.09                         & 31.31                         & 77.46                         & 58.99                         & 65.02                         & 75.73                         & 73.13                         \\
                                   & \textbf{MMRL++} & 69.91                         & 93.75                         & 96.94                         & 71.55                         & 75.15                         & 91.79                         & 31.42                         & 77.03                         & 59.60                         & 65.55                         & 76.95                         & 73.60                         \\\cmidrule(l){2-14}
\multirow{-8}{*}{\textbf{8-shot}}  & \textbf{\emph{A$_3$B$_2$}}   & \cellcolor[HTML]{FFCCC9}69.97 & \cellcolor[HTML]{FFCCC9}94.25 & \cellcolor[HTML]{FFCCC9}97.68 & \cellcolor[HTML]{FFCCC9}71.75 & \cellcolor[HTML]{FFCCC9}75.77 & \cellcolor[HTML]{FFCCC9}92.01 & \cellcolor[HTML]{FFCCC9}31.93 & \cellcolor[HTML]{FFCCC9}77.25 & \cellcolor[HTML]{FFCCC9}59.86 & \cellcolor[HTML]{FFCCC9}65.28 & \cellcolor[HTML]{FFCCC9}76.90 & \cellcolor[HTML]{FFCCC9}73.88 \\ \midrule
                                   & \textbf{CoOp}   & 70.47                         & 94.18                         & 97.78                         & 71.70                         & 71.94                         & 91.25                         & 33.31                         & 75.43                         & 50.48                         & 59.38                         & 74.17                         & 71.83                         \\
                                   & \textbf{KgCoOp} & 70.59                         & 94.65                         & 97.43                         & 73.35                         & 72.88                         & 91.50                         & 33.29                         & 74.97                         & 56.88                         & 58.89                         & 74.36                         & 72.62                         \\
                                   & \textbf{MaPLe}  & 70.49                         & 94.14                         & 97.93                         & 73.33                         & 72.03                         & 91.54                         & 34.35                         & 76.93                         & 55.63                         & 70.13                         & 78.06                         & 74.05                         \\
                                   & \textbf{TCP}    & 70.21                         & 94.54                         & 97.59                         & 75.40                         & 75.25                         & 91.53                         & 35.19                         & 78.11                         & 53.90                         & 70.17                         & 77.93                         & 74.53                         \\
                                   & \textbf{MMA}    & 69.63                         & 93.63                         & 97.59                         & 72.53                         & 75.96                         & 90.51                         & 35.39                         & 77.75                         & 65.05                         & 62.94                         & 77.73                         & 74.43                         \\
                                   & \textbf{MMRL}   & 70.70                         & 94.07                         & 97.26                         & 74.99                         & 75.13                         & 91.82                         & 37.51                         & 79.16                         & 62.04                         & 68.44                         & 78.35                         & 75.41                         \\
                                   & \textbf{MMRL++} & 70.72                         & 94.36                         & 97.60                         & 74.83                         & 76.59                         & 92.35                         & 37.16                         & 79.05                         & 62.57                         & 68.73                         & 79.57                         & 75.78                         \\\cmidrule(l){2-14}
\multirow{-8}{*}{\textbf{16-shot}} & \textbf{\emph{A$_3$B$_2$}}   & \cellcolor[HTML]{FFCCC9}70.83 & \cellcolor[HTML]{FFCCC9}95.17 & \cellcolor[HTML]{FFCCC9}97.83 & \cellcolor[HTML]{FFCCC9}75.58 & \cellcolor[HTML]{FFCCC9}77.39 & \cellcolor[HTML]{FFCCC9}91.75 & \cellcolor[HTML]{FFCCC9}37.48 & \cellcolor[HTML]{FFCCC9}78.72 & \cellcolor[HTML]{FFCCC9}64.71 & \cellcolor[HTML]{FFCCC9}69.75 & \cellcolor[HTML]{FFCCC9}79.58 & \cellcolor[HTML]{FFCCC9}76.25 \\ \bottomrule[1.5pt]
\end{tabular}%
}
\vspace{-0.3cm}
\caption{Comparison (Novel) of \emph{A$_3$B$_2$} with 7 leading methods on few-shot learning across 11 datasets.}
\label{tab:shot-novel}
\end{table*}

\begin{table*}[t]
\resizebox{\textwidth}{!}{%
\begin{tabular}{c|c|cccccccccccc}
\toprule[1.5pt]
\midrule
\textbf{Setup}                     & \textbf{Method} & \textbf{ImageNet}             & \textbf{Caltech101}           & \textbf{OxfordPets}           & \textbf{StanfordCars}         & \textbf{Flowers102}           & \textbf{Food101}              & \textbf{FGVCAircraft}         & \textbf{SUN397}               & \textbf{DTD}                  & \textbf{EuroSAT}              & \textbf{UCF101}               & \textbf{Average}              \\ \midrule
                                   & \textbf{CoOp}   & 67.56                         & 92.75                         & 93.43                         & 59.00                         & 66.73                         & 87.89                         & 17.44                         & 68.98                         & 44.51                         & 46.36                         & 68.11                         & 65.20                         \\
                                   & \textbf{KgCoOp} & 67.58                         & 92.67                         & 93.29                         & 58.87                         & 67.35                         & 88.10                         & 17.24                         & 68.60                         & 47.74                         & 46.15                         & 68.15                         & 65.42                         \\
                                   & \textbf{MaPLe}  & 65.00                         & 92.71                         & 92.68                         & 59.56                         & 71.04                         & 85.77                         & 20.21                         & 67.99                         & 48.13                         & 56.22                         & 68.71                         & 66.41                         \\
                                   & \textbf{TCP}    & 64.79                         & 92.58                         & 92.29                         & 58.58                         & 71.58                         & 85.91                         & 19.56                         & 67.73                         & 46.64                         & 55.96                         & 67.89                         & 65.97                         \\
                                   & \textbf{MMA}    & 69.24                         & 92.59                         & 94.32                         & 59.87                         & 72.95                         & 88.06                         & 21.21                         & 69.20                         & 52.15                         & 45.92                         & 70.25                         & 67.16                         \\
                                   & \textbf{MMRL}   & 69.30                         & 93.16                         & 93.28                         & 59.83                         & 74.25                         & 88.88                         & 18.87                         & 71.38                         & 52.09                         & 57.58                         & 70.94                         & 68.28                         \\
                                   & \textbf{MMRL++} & 70.10                         & 93.61                         & 93.96                         & 59.51                         & 73.56                         & 88.50                         & 18.48                         & 71.37                         & 52.72                         & 57.84                         & 70.89                         & 68.36                         \\ \cmidrule(l){2-14}
\multirow{-8}{*}{\textbf{1-shot}}  & \textbf{\emph{A$_3$B$_2$}}   & \cellcolor[HTML]{FFCCC9}69.98 & \cellcolor[HTML]{FFCCC9}93.68 & \cellcolor[HTML]{FFCCC9}94.56 & \cellcolor[HTML]{FFCCC9}61.33 & \cellcolor[HTML]{FFCCC9}73.82 & \cellcolor[HTML]{FFCCC9}88.81 & \cellcolor[HTML]{FFCCC9}21.43 & \cellcolor[HTML]{FFCCC9}71.54 & \cellcolor[HTML]{FFCCC9}52.71 & \cellcolor[HTML]{FFCCC9}57.46 & \cellcolor[HTML]{FFCCC9}71.23 & \cellcolor[HTML]{FFCCC9}68.92 \\ \midrule
                                   & \textbf{CoOp}   & 68.33                         & 93.71                         & 94.32                         & 62.20                         & 74.34                         & 87.45                         & 21.38                         & 69.69                         & 50.04                         & 53.12                         & 70.24                         & 67.95                         \\
                                   & \textbf{KgCoOp} & 68.35                         & 93.84                         & 93.93                         & 61.66                         & 74.59                         & 87.61                         & 21.13                         & 69.31                         & 50.59                         & 53.15                         & 70.26                         & 67.91                         \\
                                   & \textbf{MaPLe}  & 66.14                         & 94.11                         & 94.52                         & 64.02                         & 75.81                         & 86.80                         & 23.37                         & 70.46                         & 51.26                         & 60.01                         & 71.34                         & 69.13                         \\
                                   & \textbf{TCP}    & 65.48                         & 94.01                         & 94.16                         & 62.46                         & 76.40                         & 86.94                         & 22.60                         & 70.22                         & 49.70                         & 60.02                         & 69.99                         & 68.57                         \\
                                   & \textbf{MMA}    & 70.45                         & 93.71                         & 95.08                         & 63.33                         & 78.82                         & 87.48                         & 23.55                         & 71.63                         & 57.03                         & 52.95                         & 72.15                         & 69.84                         \\
                                   & \textbf{MMRL}   & 70.62                         & 93.82                         & 94.00                         & 63.52                         & 78.78                         & 88.56                         & 22.62                         & 74.11                         & 57.96                         & 60.96                         & 73.26                         & 70.94                         \\
                                   & \textbf{MMRL++} & 71.23                         & 94.46                         & 94.18                         & 63.43                         & 78.41                         & 88.38                         & 21.33                         & 73.52                         & 58.39                         & 60.68                         & 73.78                         & 70.88                         \\ \cmidrule(l){2-14}
\multirow{-8}{*}{\textbf{2-shot}}  & \textbf{\emph{A$_3$B$_2$}}   & \cellcolor[HTML]{FFCCC9}71.19 & \cellcolor[HTML]{FFCCC9}94.56 & \cellcolor[HTML]{FFCCC9}95.39 & \cellcolor[HTML]{FFCCC9}65.44 & \cellcolor[HTML]{FFCCC9}79.15 & \cellcolor[HTML]{FFCCC9}88.19 & \cellcolor[HTML]{FFCCC9}25.34 & \cellcolor[HTML]{FFCCC9}72.99 & \cellcolor[HTML]{FFCCC9}57.63 & \cellcolor[HTML]{FFCCC9}61.11 & \cellcolor[HTML]{FFCCC9}73.88 & \cellcolor[HTML]{FFCCC9}71.53 \\ \midrule
                                   & \textbf{CoOp}   & 70.20                         & 94.40                         & 96.26                         & 65.70                         & 78.45                         & 88.03                         & 25.18                         & 73.30                         & 51.65                         & 67.77                         & 74.32                         & 71.69                         \\
                                   & \textbf{KgCoOp} & 70.03                         & 94.60                         & 95.87                         & 65.13                         & 78.73                         & 88.14                         & 24.85                         & 72.90                         & 55.36                         & 66.55                         & 74.34                         & 71.72                         \\
                                   & \textbf{MaPLe}  & 68.76                         & 94.57                         & 95.58                         & 67.35                         & 79.02                         & 87.12                         & 26.38                         & 74.23                         & 56.33                         & 73.03                         & 75.05                         & 72.79                         \\
                                   & \textbf{TCP}    & 68.06                         & 94.49                         & 95.22                         & 65.66                         & 79.61                         & 87.22                         & 25.54                         & 73.97                         & 54.55                         & 70.44                         & 73.63                         & 71.92                         \\
                                   & \textbf{MMA}    & 71.08                         & 94.17                         & 95.34                         & 67.46                         & 81.15                         & 87.59                         & 27.75                         & 74.60                         & 64.18                         & 69.73                         & 75.86                         & 73.80                         \\
                                   & \textbf{MMRL}   & 71.72                         & 94.98                         & 95.03                         & 68.08                         & 81.72                         & 88.83                         & 26.75                         & 76.57                         & 64.07                         & 73.12                         & 77.17                         & 74.62                         \\
                                   & \textbf{MMRL++} & 71.58                         & 95.34                         & 95.47                         & 67.91                         & 81.85                         & 89.19                         & 26.90                         & 76.05                         & 64.36                         & 74.23                         & 77.78                         & 74.84                         \\ \cmidrule(l){2-14}
\multirow{-8}{*}{\textbf{4-shot}}  & \textbf{\emph{A$_3$B$_2$}}   & \cellcolor[HTML]{FFCCC9}71.82 & \cellcolor[HTML]{FFCCC9}95.01 & \cellcolor[HTML]{FFCCC9}95.59 & \cellcolor[HTML]{FFCCC9}69.96 & \cellcolor[HTML]{FFCCC9}82.16 & \cellcolor[HTML]{FFCCC9}89.33 & \cellcolor[HTML]{FFCCC9}29.91 & \cellcolor[HTML]{FFCCC9}76.28 & \cellcolor[HTML]{FFCCC9}64.00 & \cellcolor[HTML]{FFCCC9}74.42 & \cellcolor[HTML]{FFCCC9}77.47 & \cellcolor[HTML]{FFCCC9}75.31 \\ \midrule
                                   & \textbf{CoOp}   & 71.95                         & 95.04                         & 95.87                         & 69.94                         & 80.86                         & 88.95                         & 31.83                         & 74.94                         & 59.87                         & 67.53                         & 76.71                         & 74.29                         \\
                                   & \textbf{KgCoOp} & 71.97                         & 95.18                         & 95.46                         & 69.34                         & 81.12                         & 89.06                         & 31.45                         & 74.51                         & 61.09                         & 66.61                         & 76.76                         & 74.18                         \\
                                   & \textbf{MaPLe}  & 71.43                         & 95.35                         & 96.14                         & 71.04                         & 81.69                         & 89.05                         & 31.79                         & 76.90                         & 61.40                         & 74.83                         & 77.82                         & 75.53                         \\
                                   & \textbf{TCP}    & 70.69                         & 95.23                         & 95.93                         & 69.26                         & 82.33                         & 88.68                         & 31.28                         & 76.57                         & 60.15                         & 72.17                         & 76.33                         & 74.68                         \\
                                   & \textbf{MMA}    & 71.85                         & 95.07                         & 95.92                         & 71.82                         & 83.75                         & 88.52                         & 33.08                         & 77.57                         & 67.85                         & 70.64                         & 79.04                         & 76.21                         \\
                                   & \textbf{MMRL}   & 72.63                         & 95.26                         & 95.53                         & 71.91                         & 83.44                         & 89.94                         & 32.27                         & 78.75                         & 67.60                         & 74.00                         & 79.03                         & 76.66                         \\
                                   & \textbf{MMRL++} & 72.81                         & 95.66                         & 95.71                         & 71.91                         & 84.24                         & 89.96                         & 32.87                         & 78.24                         & 68.02                         & 74.57                         & 79.69                         & 76.95                         \\ \cmidrule(l){2-14}
\multirow{-8}{*}{\textbf{8-shot}}  & \textbf{\emph{A$_3$B$_2$}}   & \cellcolor[HTML]{FFCCC9}72.99 & \cellcolor[HTML]{FFCCC9}95.97 & \cellcolor[HTML]{FFCCC9}96.35 & \cellcolor[HTML]{FFCCC9}74.20 & \cellcolor[HTML]{FFCCC9}84.74 & \cellcolor[HTML]{FFCCC9}90.24 & \cellcolor[HTML]{FFCCC9}34.91 & \cellcolor[HTML]{FFCCC9}78.99 & \cellcolor[HTML]{FFCCC9}68.05 & \cellcolor[HTML]{FFCCC9}75.56 & \cellcolor[HTML]{FFCCC9}80.12 & \cellcolor[HTML]{FFCCC9}77.73 \\ \midrule
                                   & \textbf{CoOp}   & 73.21                         & 96.22                         & 96.51                         & 73.26                         & 82.65                         & 90.75                         & 35.42                         & 78.22                         & 61.49                         & 71.17                         & 78.65                         & 76.59                         \\
                                   & \textbf{KgCoOp} & 73.23                         & 96.37                         & 96.12                         & 72.63                         & 82.93                         & 90.87                         & 34.98                         & 77.79                         & 65.90                         & 69.85                         & 78.68                         & 76.65                         \\
                                   & \textbf{MaPLe}  & 73.48                         & 82.54                         & 96.55                         & 74.72                         & 82.70                         & 90.91                         & 36.62                         & 79.31                         & 65.86                         & 79.82                         & 81.33                         & 78.27                         \\
                                   & \textbf{TCP}    & 72.08                         & 96.05                         & 96.17                         & 72.85                         & 83.35                         & 91.05                         & 35.49                         & 79.02                         & 63.81                         & 76.57                         & 79.80                         & 77.26                         \\
                                   & \textbf{MMA}    & 73.22                         & 96.02                         & 96.38                         & 75.60                         & 85.50                         & 89.69                         & 38.90                         & 79.58                         & 73.33                         & 75.47                         & 81.75                         & 78.99                         \\
                                   & \textbf{MMRL}   & 73.72                         & 96.09                         & 96.33                         & 75.27                         & 84.98                         & 91.17                         & 38.10                         & 80.49                         & 71.09                         & 77.90                         & 81.78                         & 79.07                         \\
                                   & \textbf{MMRL++} & 73.67                         & 96.27                         & 96.36                         & 75.21                         & 85.86                         & 91.53                         & 38.31                         & 80.28                         & 71.42                         & 78.18                         & 82.42                         & 79.29                         \\ \cmidrule(l){2-14}
\multirow{-8}{*}{\textbf{16-shot}} & \textbf{\emph{A$_3$B$_2$}}   & \cellcolor[HTML]{FFCCC9}73.96 & \cellcolor[HTML]{FFCCC9}96.92 & \cellcolor[HTML]{FFCCC9}96.66 & \cellcolor[HTML]{FFCCC9}78.64 & \cellcolor[HTML]{FFCCC9}86.54 & \cellcolor[HTML]{FFCCC9}90.46 & \cellcolor[HTML]{FFCCC9}41.93 & \cellcolor[HTML]{FFCCC9}80.50 & \cellcolor[HTML]{FFCCC9}73.17 & \cellcolor[HTML]{FFCCC9}80.75 & \cellcolor[HTML]{FFCCC9}82.89 & \cellcolor[HTML]{FFCCC9}80.49 \\ \bottomrule[1.5pt]
\end{tabular}%
}
\vspace{-0.3cm}
\caption{Comparison (HM) of \emph{A$_3$B$_2$} with 7 leading methods on few-shot learning across 11 datasets.}
\label{tab:shot-hm}
\end{table*}

\section{Fixed Asymmetric Structure}
\label{sec:fixed}

To better analyze the asymmetric design of \emph{A$_3$B$_2$}, we remove $\mathcal{L}_{bias}$ and $\mathcal{L}_{bal}$ and denote the resulting variant as A3.
\subsection{Empirical Evidence}
\label{sec:empirical}
We attempt to invert A3 into $\overline{A3}$ (\textit{i.e.}, multiple single-dimensionality reduction matrices and a dimensionality expansion matrix). We compare the performance of $\overline{A3}$ and A3 across all image classification tasks, as shown in Figures~\ref{fig:inverse_base},~\ref{fig:inverse_novel},~\ref{fig:inverse_hm},~\ref{fig:inverse_cross}, and~\ref{fig:inverse_domain}. From these results, we observe that \textbf{$\overline{A3}$ generally performs worse than A3 across different tasks}. This strongly demonstrates the effectiveness of the proposed fixed asymmetric design. In the following, we analyze the underlying reasons behind this outcome.

\begin{figure*}
\centering
    \includegraphics[width=0.9\linewidth]{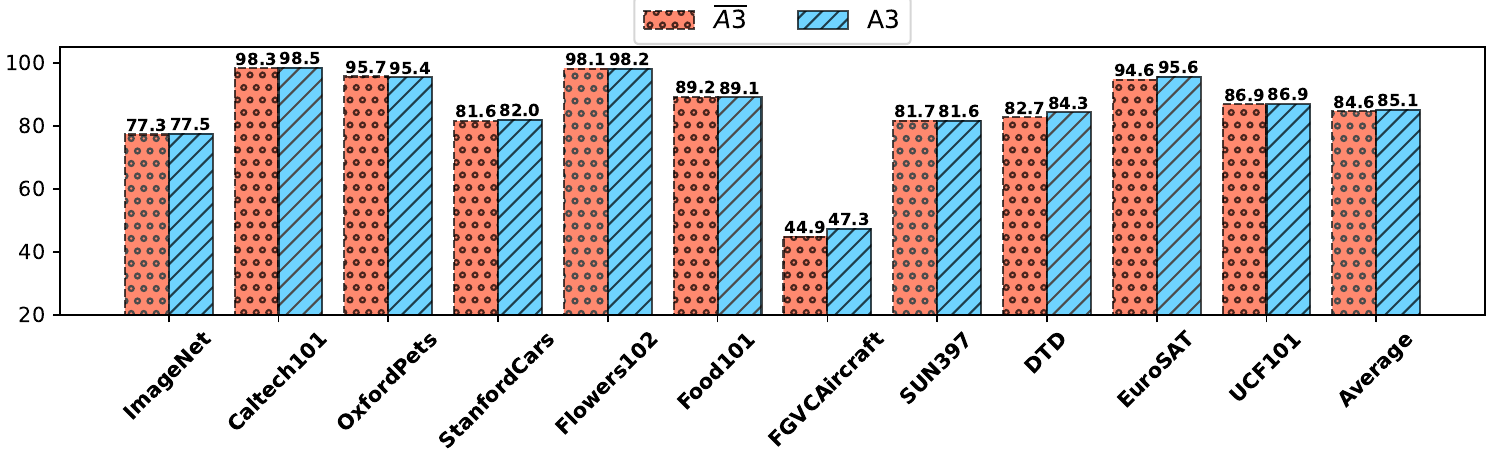}
    \vspace{-0.2cm}
    \textbf{\caption{Performance comparison of $\overline{A3}$ and A3 in terms of the base setting in base-to-novel generalization.}
    \label{fig:inverse_base}}
    \vspace{-0.4cm}
\end{figure*}

\begin{figure*}
\centering
    \includegraphics[width=0.9\linewidth]{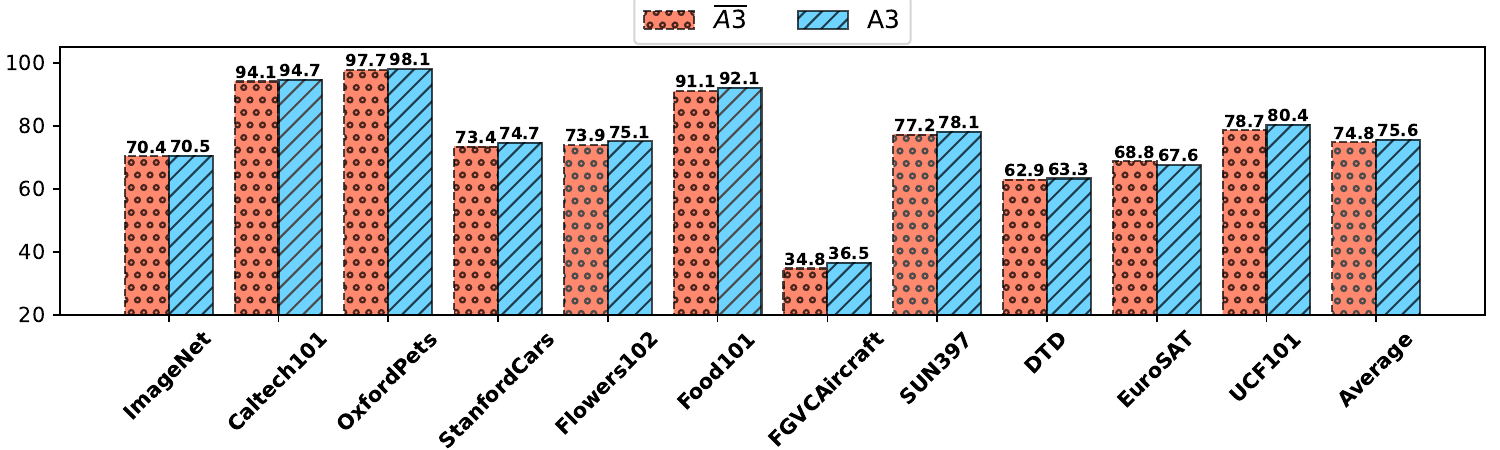}
    \vspace{-0.2cm}
    \textbf{\caption{Performance comparison of $\overline{A3}$ and A3 in terms of the novel setting in base-to-novel generalization.}
    \label{fig:inverse_novel}}
    \vspace{-0.4cm}
\end{figure*}

\begin{figure*}
\centering
    \includegraphics[width=0.9\linewidth]{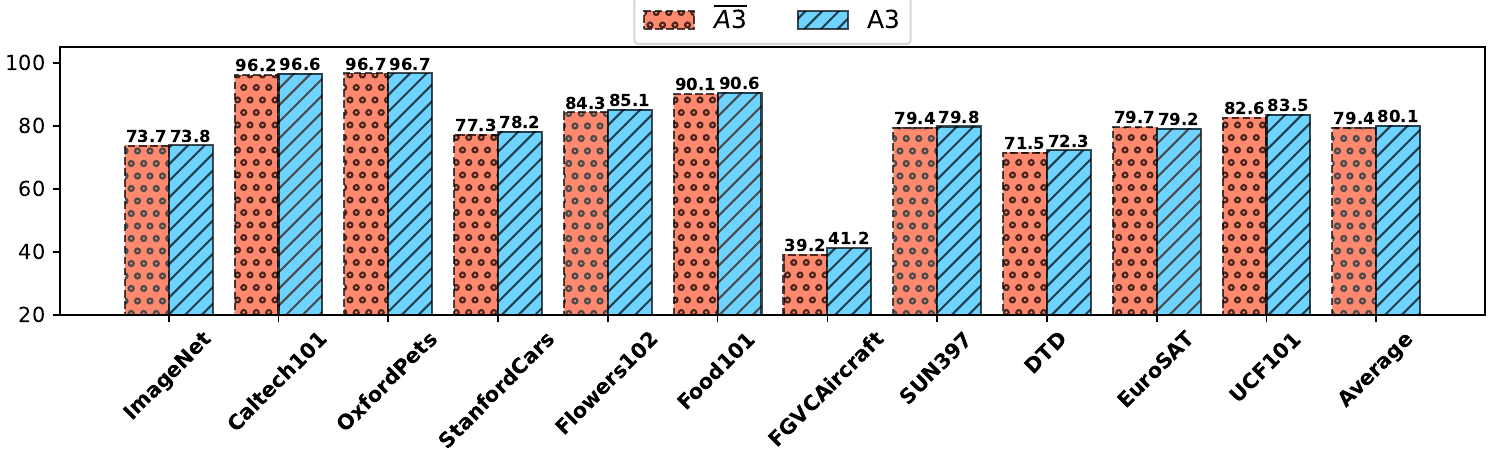}
    \vspace{-0.2cm}
    \textbf{\caption{Performance comparison of $\overline{A3}$ and A3 in terms of the hm setting in base-to-novel generalization.}
    \label{fig:inverse_hm}}
    \vspace{-0.4cm}
\end{figure*}

\begin{figure*}
\centering
    \includegraphics[width=0.9\linewidth]{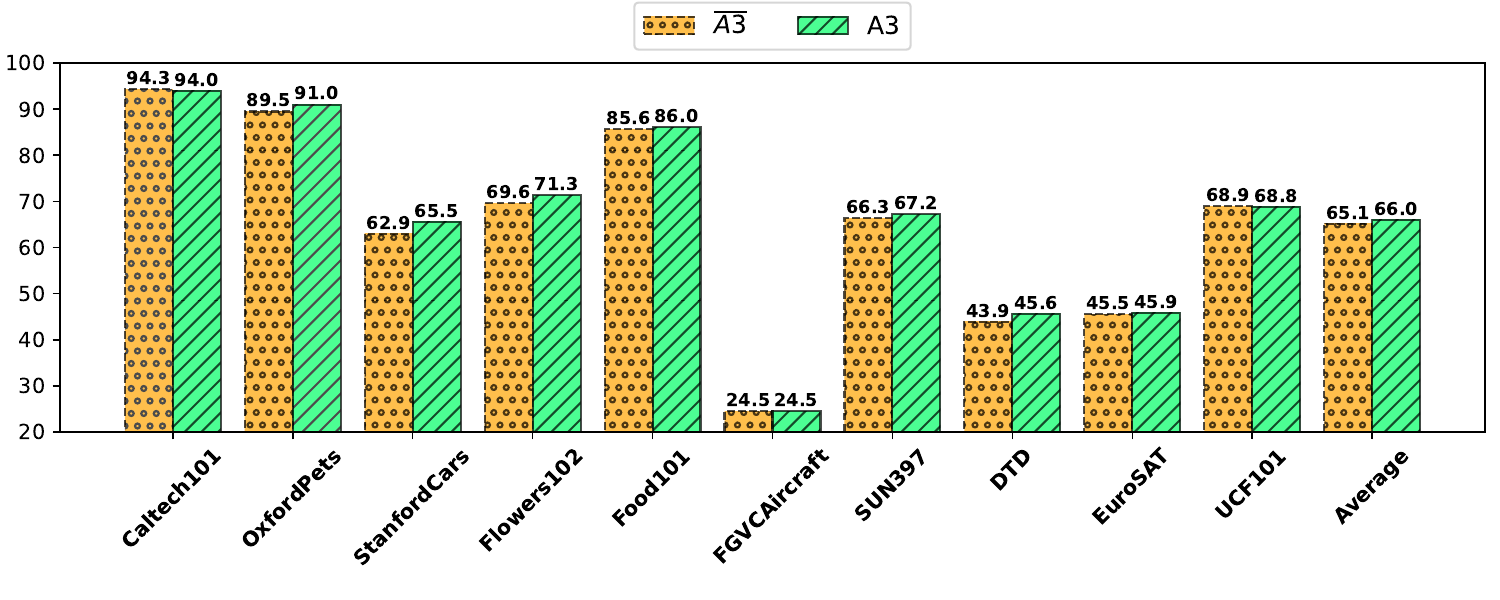}
    \vspace{-0.5cm}
    \textbf{\caption{Performance comparison of $\overline{A3}$ and A3 in cross-dataset evaluation.}
    \label{fig:inverse_cross}}
    \vspace{-0.4cm}
\end{figure*}

\begin{figure}[t]
\centering
    \includegraphics[width=0.9\linewidth]{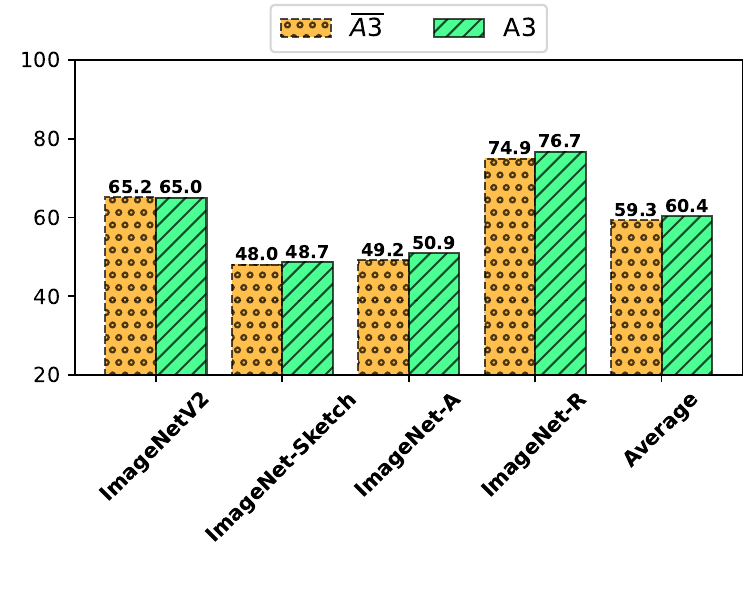}
    \vspace{-0.9cm}
    \textbf{\caption{Performance comparison of $\overline{A3}$ and A3 in domain generalization.}
    \label{fig:inverse_domain}}
    \vspace{-0.4cm}
\end{figure}

\subsection{Theoretical Support}
\label{sec:theoretical}

We build upon the theoretical analysis developed in our previous work~\cite{zhou2026beyond}. A critical design choice arises when integrating this low-rank factorization with MoE principles to handle heterogeneous data distributions. We analyze two possible asymmetric architectures:
\begin{itemize}[leftmargin=*, topsep=2pt, itemsep=2pt]
    \item \textbf{One-Down-Many-Ups}: A single, shared dimension-reducing projection $W^{\text{down}}$ feeds into multiple specialized dimension-increasing expert heads $\{W^{\text{up}}_h\}_{h=1}^H$.
    \item \textbf{Many-Downs-One-Up}: Multiple specialized projections $\{W^{\text{down}}_h\}_{h=1}^H$ are aggregated and passed through a single shared reconstruction layer $W^{\text{up}}$.
\end{itemize}
We found that the first configuration is theoretically superior, and we ground our analysis in the Information Bottleneck (IB) principle~\cite{koch1987shifts}.

The IB principle frames learning as a problem of compression. Given an input random variable $X$ and a target variable $Y$, the goal is to learn a compressed representation, or bottleneck, $Z$, that is maximally informative about $Y$ while being minimally informative about $X$. This trade-off is formalized by minimizing the Lagrangian:
\begin{equation}
    \mathcal{L}_{\text{IB}}(p(z|x)) = I(X; Z) - \beta I(Z; Y)
    \label{eq:ib_lagrangian}
\end{equation}
where $I(\cdot; \cdot)$ denotes mutual information and $\beta$ is a Lagrange multiplier controlling the trade-off between compression ($I(X;Z)$) and prediction ($I(Z;Y)$). The optimal representation $Z$ is thus a \textbf{minimal sufficient statistic} of $X$ for predicting $Y$.

\subsubsection{Information-Theoretic Analysis of Asymmetric Architectures}

We now apply the IB framework to analyze the two competing architectures. Let the input features be $X$ and the final output be $Y$.

\paragraph{The One-Down-Many-Ups Configuration as a Unified Bottleneck}

This architecture, as defined in $A_3B_2$, is given by:
\begin{equation}
    Y = \sum_{h=1}^{H} G_h(X) \cdot W^{\text{up}}_h\left(W^{\text{down}}(X)\right)
    \label{eq:one_down}
\end{equation}
Let us define the bottleneck variable as the output of the shared projection: $Z \triangleq W^{\text{down}}(X)$. The IB objective for this architecture is to learn the parameters of $W^{\text{down}}$ (which define the mapping $p(z|x)$) that minimize $\mathcal{L}_{\text{IB}}$ from Eq.~\ref{eq:ib_lagrangian}.

\paragraph{Theoretical Analysis.}
The one-down-many-ups architecture imposes a single shared bottleneck: all information from $X$ to $Y$ must pass through the same low-dimensional $Z$. This means $Z$ must serve as the representation for the \emph{entire mixture} of $H$ experts. Consequently, to maximize the predictive information $I(Z; Y)$, $Z$ is forced to encode only those features of $X$ that are salient for predicting $Y$ across all expert domains, while filtering out any spurious or task-specific details~\cite{ye2024spurious,zhou2025disentangled,zhou2025multimodal}. In effect, the model learns to focus on the joint task distribution $p(y|x) = \sum_{h=1}^H p(h|x)\,p(y|x,h)$, where $p(h|x)$ is given by the gating function $G_h(X)$. Since the architecture essentially forms a Markov chain $X \to Z \to Y$~\cite{norris1998markov}, the Data Processing Inequality guarantees that $I(Z; Y) \le I(X; Y)$. The IB objective drives $Z$ to approach this upper bound on predictive power while minimizing $I(X; Z)$, yielding a minimal sufficient representation of $X$ for the entire problem.

This unified bottleneck encourages $Z$ to become a robust, canonical representation of shared knowledge that all experts can rely on. Each expert head $W^{\text{up}}_h$ then performs a simpler, decoupled task: mapping the common $r$-dimensional representation $Z$ to a specialized output subspace in $\mathbb{R}^n$. Notably, although $Z$ has dimension $r$, the one-down-many-ups design is not limited to a single rank-$r$ mapping. Because each $W^{\text{up}}_h$ can independently span an output subspace of dimension up to $r$, the combined model can cover the union of these subspaces. Through gating, different inputs activate different experts, allowing the overall output $Y$ to lie in a rich space of dimension up to roughly $rH$, far greater than any single low-rank projector could achieve. In other words, the architecture dedicates most of its parameters to the expert-specific “up” transformations, thereby boosting expressiveness where needed while keeping the shared encoder compact.

Another major advantage of the unified bottleneck is in gradient flow during training. The error gradients from all expert outputs are propagated back into the single projection $W^{\text{down}}$. This shared pathway means that every expert's feedback refines the same encoder $W^{\text{down}}$, aligning the experts on learning a common representation. Such unified credit assignment avoids conflicting updates—improvements for one expert cannot come at the expense of degrading the shared features for another~\cite{samejima2003inter,mu2025comprehensive}. This synergy makes optimization more stable and efficient.

\paragraph{The Many-Downs-One-Up Configuration as Competing Bottlenecks}

The inverse architecture is formulated as:
\begin{equation}
    Y = W^{\text{up}}\left(\sum_{h=1}^{H} G_h(X) \cdot W^{\text{down}}_h(X)\right)
    \label{eq:many_downs}
\end{equation}
Here, we have $H$ distinct bottleneck variables, $Z_h \triangleq W^{\text{down}}_h(X)$, which are aggregated to form an intermediate representation, $Z_{\text{agg}} \triangleq \sum_{h=1}^{H} G_h(X) Z_h$, before the final reconstruction.

\paragraph{Theoretical Analysis.}
The many-downs-one-up configuration is fundamentally limited from an information-theoretic perspective. According to the Data Processing Inequality, for any Markov chain $A \to B \to C$, we have $I(A; C) \le I(A; B)$. In our case, the processing chain of interest is:
\begin{equation}
    (Z_1, \dots, Z_H) \to Z_{\text{agg}} \to Y
\end{equation}
Therefore, $I((Z_1, \dots, Z_H); Y) \ge I(Z_{\text{agg}}; Y)$. The aggregation $Z_{\text{agg}} = \sum_{h=1}^H G_h(X)\, Z_h$ is a \emph{lossy} operation that conflates information from the experts. It mixes specialized features from different expert channels into a single vector before reconstruction, making it impossible for the shared $W^{\text{up}}$ layer to discern the contribution of each expert. Moreover, this architecture imposes a strict capacity bottleneck: all expert outputs must funnel through the same rank-$r$ reconstruction matrix $W^{\text{up}}$. As $W^{\text{up}}$ can only produce outputs in an $r$-dimensional subspace of $\mathbb{R}^n$, the final prediction $Y$ is confined to a fixed low-dimensional subspace regardless of how diverse the individual $Z_h$ are. In contrast to the one-down-many-ups design (where multiple $W^{\text{up}}_h$ heads could span different output subspaces), here the presence of $H$ distinct projections $W^{\text{down}}_h$ does not increase the overall representational rank—any benefit of having multiple experts is squandered by the single shared up-projection.

Training this architecture also poses a severe credit assignment problem~\cite{fedus2022switch,zhang2025more}. The single reconstruction layer $W^{\text{up}}$ must learn to invert a complex mixture $Z_{\text{agg}}$ back to $Y$ without access to the individual expert outputs $Z_h$. During backpropagation, the gradient of the loss with respect to $Z_{\text{agg}}$ is shared by all experts. By the chain rule one can show $\frac{\partial \mathcal{L}}{\partial Z_h} = G_h(X) \frac{\partial \mathcal{L}}{\partial Z_{\text{agg}}}$ for each expert $h$. Thus, each $W^{\text{down}}_h$ only receives a fraction $G_h(X)$ of the total feedback signal, weighted by its gating activation. When multiple experts are active, their parameter updates become entangled and often conflicting, since the model cannot clearly know which expert was responsible for which aspect of the prediction error. This ambiguity in credit assignment leads to unstable optimization and hampers the learning of effective specialized representations. In practice, the easiest path for the model is to collapse the experts into an effectively single expert to eliminate the ambiguity. For example, it may drive all $Z_h$ to become nearly identical or favor one expert exclusively, so that $W^{\text{up}}$ deals with a single dominant bottleneck. Such outcomes defeat the purpose of having $H$ experts in the first place, confirming the fundamental difficulty of the many-downs-one-up design.

\subsubsection{Conclusion of Architectural Design}

The ``one-down-many-ups'' architectural design is not just empirically successful but also theoretically sound when analyzed through the lens of the Information Bottleneck principle. It creates a unified optimization objective that promotes the learning of a minimal, sufficient, and shared representation. In contrast, the ``many-downs-one-up'' structure introduces information loss through premature aggregation and creates a challenging credit assignment problem, hindering effective specialization. Therefore, we conclude that enforcing a single, shared information bottleneck before expert-level specialization is a superior design principle for building efficient and robust low-rank models.

\begin{table*}[th!]
\centering
\setlength{\tabcolsep}{0.00001mm}
\begin{tabular}{@{}c|c|cccccccccc|c@{}}
\toprule[1.5pt]
\midrule
\rotatebox{40}{\textbf{Method}}       & \rotatebox{40}{\textbf{ImageNet}} & \rotatebox{40}{\textbf{Caltech101}} & \rotatebox{40}{\textbf{OxfordPets}} & \rotatebox{40}{\textbf{StanfordCars}} & \rotatebox{40}{\textbf{Flowers102}} & \rotatebox{40}{\textbf{Food101}} & \rotatebox{40}{\textbf{FGVCAircraft}} & \rotatebox{40}{\textbf{SUN397}} & \rotatebox{40}{\textbf{DTD}} & \rotatebox{40}{\textbf{EuroSAT}} & \rotatebox{40}{\textbf{UCF101}} & \rotatebox{40}{\textbf{Average}} \\ \midrule 
\textbf{CoCoOp}       & 70.62             & 94.52               & 90.47               & 65.91                 & 71.92               & 86.02            & 23.34                 & 66.54           & 45.51        & 44.43            & 68.24           & 65.69            \\
\textbf{CLIP-Adapter} & 68.64             & 93.50               & 88.91               & 64.90                 & 70.48               & 85.88            & 24.66                 & 65.50           & 45.21        & 48.22            & 66.79           & 65.41            \\
\textbf{RPO}          & 69.69             & 93.75               & 89.13               & 66.25                 & 70.60               & 85.07            & 24.25                 & 66.95           & 45.86        & 43.95            & 67.29           & 65.31            \\ \bottomrule[1.5pt]
\end{tabular}%
\vspace{-0.3cm}
\caption{Performance of three methods in cross-dataset evaluation.}
\label{tab:cross_other}
\end{table*}

\begin{table}[th!]
\centering
\setlength{\tabcolsep}{0.35mm}
\begin{tabular}{@{}c|c|cccc|c@{}}
\toprule[1.5pt]
\midrule
\textbf{Method}       & \textbf{ImageNet} & \textbf{-V2} & \textbf{-S} & \textbf{-A} & \textbf{-R} & \textbf{Average} \\ \midrule
\textbf{CoCoOp}       & 70.62             & 49.14        & 51.62       & 76.84       & 76.69       & 63.57            \\
\textbf{CLIP-Adapter} & 68.64             & 47.56        & 48.94       & 75.39       & 76.60       & 62.12            \\
\textbf{RPO}          & 69.69             & 49.18        & 50.67       & 76.83       & 76.47       & 63.29            \\ \bottomrule[1.5pt]
\end{tabular}%
\vspace{-0.3cm}
\caption{Performance of three methods in domain generalization.}
\label{tab:domain_other}
\end{table}

\section{Theoretical Guarantees of the Proposed Method}
\label{sec:guarantee}

The concept of \emph{Branch Bias} is inspired by \emph{Data Bias} in data mining. \emph{Data Bias} refers to the unequal contribution of different data components to a task (caused by distribution imbalance), and can be defined as the deviation between the actual and ideal distributions:
\[
\operatorname{Data\ Bias} = D\big(P(Y), P^{\text{ideal}}(Y)\big),
\]
where $D(\cdot)$ can be a divergence measure (\textit{e.g.}, KL divergence).

\emph{Branch Bias} differs fundamentally from \emph{Data Bias}:

\begin{table}[h]
\centering
\vspace{-0.4cm}
\resizebox{\columnwidth}{!}{
\begin{tabular}{c|c}
\toprule[1.5pt]
\textbf{Data Bias} & \textbf{Branch Bias} \\
\midrule
Imbalanced data proportion & Imbalanced representation contribution \\ \midrule
Different class weights & Different encoder importance \\ \bottomrule[1.5pt]
\end{tabular}
}
\vspace{-0.3cm}
\end{table}

We now formally define \emph{Branch Bias} and its theoretical connection to the proposed UAAD mechanism.

\begin{definition}[Branch Contribution]
Let $\mathcal{T}$ denote a downstream few-shot task with data distribution $(x,y)\sim \mathcal{T}$. Given a vision-language model with image encoder $V$, text encoder $T$, and loss $\ell$, the branch contributions are defined as:
\[
\begin{aligned}
C_V(\mathcal{T}) &= \mathbb{E}_{(x,y)\sim\mathcal{T}}\!\left[\left\|\nabla_{V(x)} \ell(f(V(x),T(y)),y)\right\|\right], \\
C_T(\mathcal{T}) &= \mathbb{E}_{(x,y)\sim\mathcal{T}}\!\left[\left\|\nabla_{T(y)} \ell(f(V(x),T(y)),y)\right\|\right],
\end{aligned}
\]
where $f(\cdot,\cdot)$ is the prediction function.
\end{definition}

\begin{definition}[\emph{Branch Bias}]
\emph{Branch Bias} for task $\mathcal{T}$ is defined as:
\[
\operatorname{Branch\ Bias}(\mathcal{T}) = D\big(P_{\mathcal{T}}, P^{\text{uniform}}\big),
\]
where
\[
P_{\mathcal{T}}=\left(\frac{C_V}{C_V+C_T}, \frac{C_T}{C_V+C_T}\right), \quad
P^{\text{uniform}}=(0.5,0.5).
\]
\end{definition}

\begin{proposition}[Performance Gain Imbalance Implies \emph{Branch Bias}]
Let $\Delta_V(\mathcal{T})$ and $\Delta_T(\mathcal{T})$ denote performance gains from adapting the image and text encoders. If
\[
\Delta_V(\mathcal{T}) \neq \Delta_T(\mathcal{T}),
\]
then
\[
\operatorname{Branch\ Bias}(\mathcal{T}) > 0.
\]
\end{proposition}

\begin{proof}
Assume $C_V \propto \Delta_V$ and $C_T \propto \Delta_T$. If $\Delta_V \neq \Delta_T$, then
\[
\frac{C_V}{C_V+C_T} \neq 0.5 \Rightarrow P_{\mathcal{T}} \neq (0.5,0.5),
\]
thus
\[
D(P_{\mathcal{T}}, P^{\text{uniform}}) > 0.
\]
\end{proof}

\begin{proposition}[Negative Transfer Amplifies \emph{Branch Bias}]
If adapting one branch causes negative transfer, \textit{e.g.}, $\Delta_V(\mathcal{T})<0$, then \emph{Branch Bias} increases.
\end{proposition}

\begin{proof}
Negative gain shifts optimal contribution toward $C_T$:
\[
\frac{C_V}{C_V+C_T} \rightarrow 0,
\]
thus increasing divergence:
\[
D(P_{\mathcal{T}}, P^{\text{uniform}}) \uparrow.
\]
\end{proof}

\begin{theorem}[Uncertainty-Aware Suppression Reduces Effective \emph{Branch Bias}]
Let $\kappa(x)$ denote prediction confidence. Consider:
\[
\mathcal{L}_{\text{bias}}=\mathbb{E}_{x\sim\mathcal{T}}\left[(1-\kappa(x)) \sum_i \|\Delta v_i(x)\|^2\right].
\]
Minimizing $\mathcal{L}_{\text{bias}}$ yields:
\[
C_V^{\mathrm{eff}}(\mathcal{T}) \le C_V(\mathcal{T}),
\]
and
\[
\operatorname{Branch\ Bias}(\mathcal{T}) \rightarrow \operatorname{Branch\ Bias}^{\text{controlled}}(\mathcal{T}).
\]
\end{theorem}

\begin{proof}
Define:
\[
\ell'(x,y)=\ell(f(V(x),T(y)),y)+(1-\kappa(x))\|\Delta v(x)\|^2.
\]
Then:
\[
\nabla_{V(x)}\ell'=\nabla_{V(x)}\ell+(1-\kappa(x))\nabla_{V(x)}\|\Delta v(x)\|^2.
\]
The second term penalizes large updates and is non-negative, hence:
\[
\|\nabla_{V(x)}\ell'\|\le \|\nabla_{V(x)}\ell\|.
\]
Taking expectation:
\[
C_V^{\text{eff}}(\mathcal{T}) \le C_V(\mathcal{T}).
\]
Thus, image-branch contribution is reduced, rebalancing contributions and lowering \emph{Branch Bias}.
\end{proof}

\textbf{Remark.} \textbf{Modality misalignment is a consequence of \emph{Branch Bias}}, while \emph{Branch Bias} is one of its underlying causes.

\section{Detailed Motivation Experiment}
\label{sec:motivation}
In this section, we present the details of our motivation experiments. Our experiments are based on MMA~\cite{yang2024mma} without shared layers, with the intermediate feature dimension set to 32 and the scaling factor set to 0.001. Figures~\ref{fig:motivation_base2new}, \ref{fig:motivation_cross_dataset}, and \ref{fig:motivation_generalization} show the results on various datasets for the base-to-novel generalization, cross-dataset evaluation, and domain generalization tasks, respectively. From Figures~\ref{fig:motivation_base2new}, \ref{fig:motivation_cross_dataset}, and \ref{fig:motivation_generalization}, we observe that \textbf{most examples follow the three insights we identified, which strongly supports their generalizability.}

\begin{figure*}[t]
    \centering
    \begin{subfigure}[b]{0.495\linewidth}
        \includegraphics[width=\linewidth]{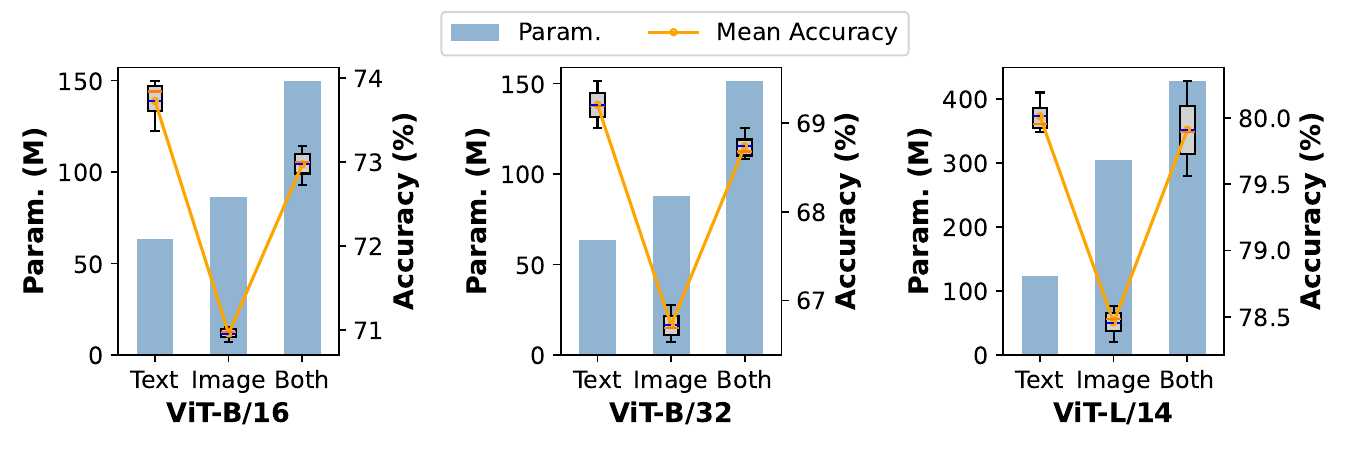}
        \vspace{-0.7cm}
        \caption{ImageNet}
    \end{subfigure}
    \hfill
    \begin{subfigure}[b]{0.495\linewidth}
        \includegraphics[width=\linewidth]{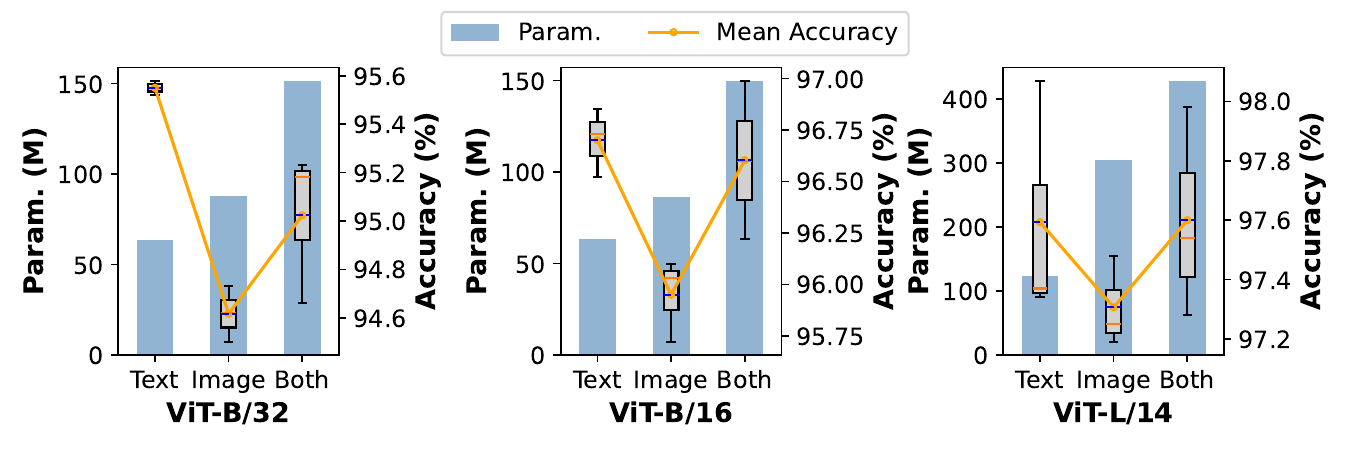}
        \vspace{-0.7cm}
        \caption{Caltech101}
    \end{subfigure}

    \begin{subfigure}[b]{0.495\linewidth}
        \includegraphics[width=\linewidth]{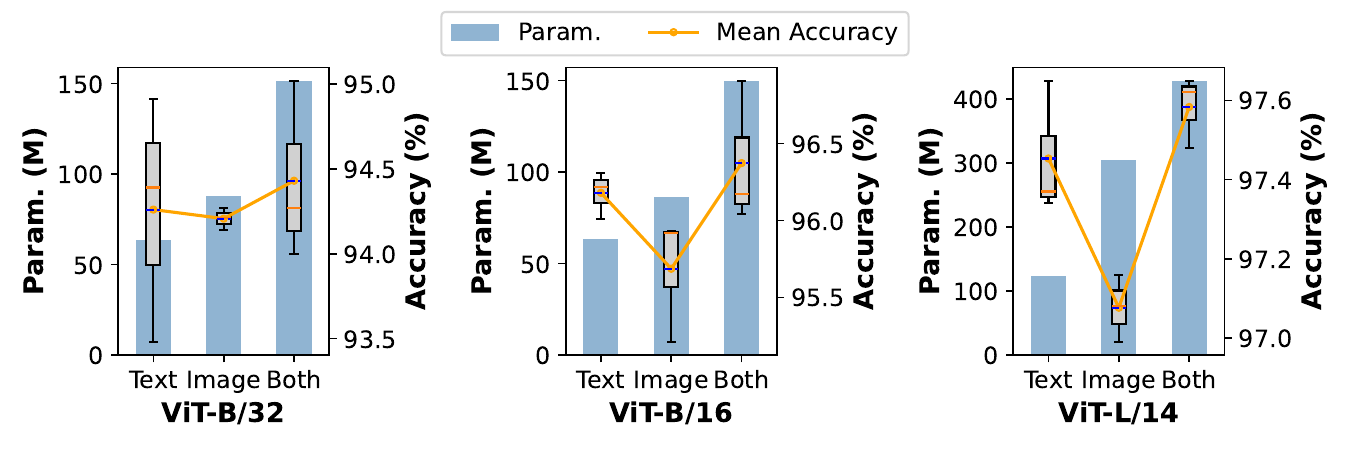}
        \vspace{-0.7cm}
        \caption{OxfordPets}
    \end{subfigure}
    \hfill
    \begin{subfigure}[b]{0.495\linewidth}
        \includegraphics[width=\linewidth]{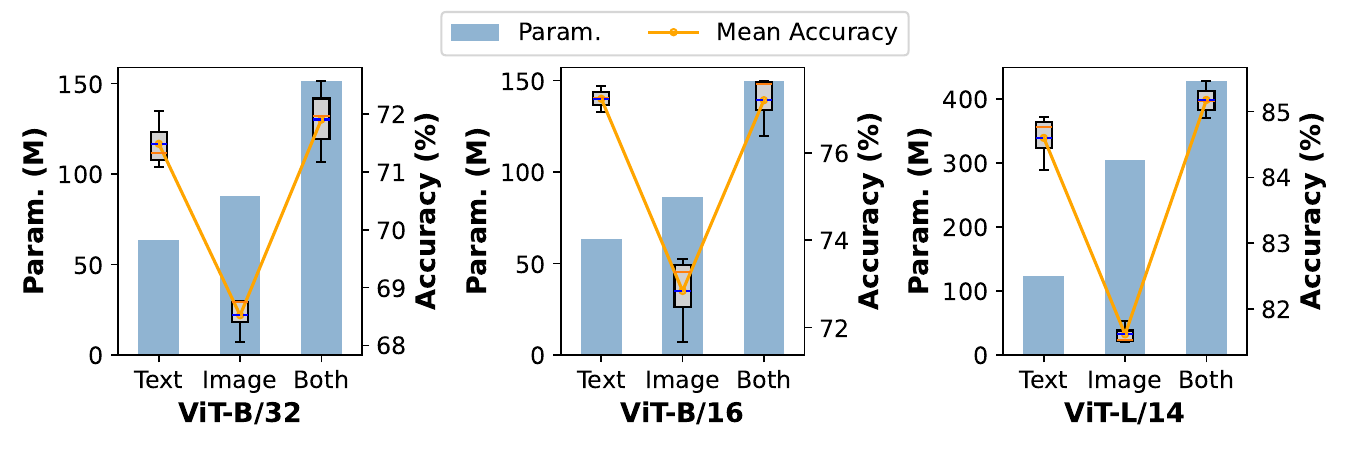}
        \vspace{-0.7cm}
        \caption{Cars}
    \end{subfigure}

    \begin{subfigure}[b]{0.495\linewidth}
        \includegraphics[width=\linewidth]{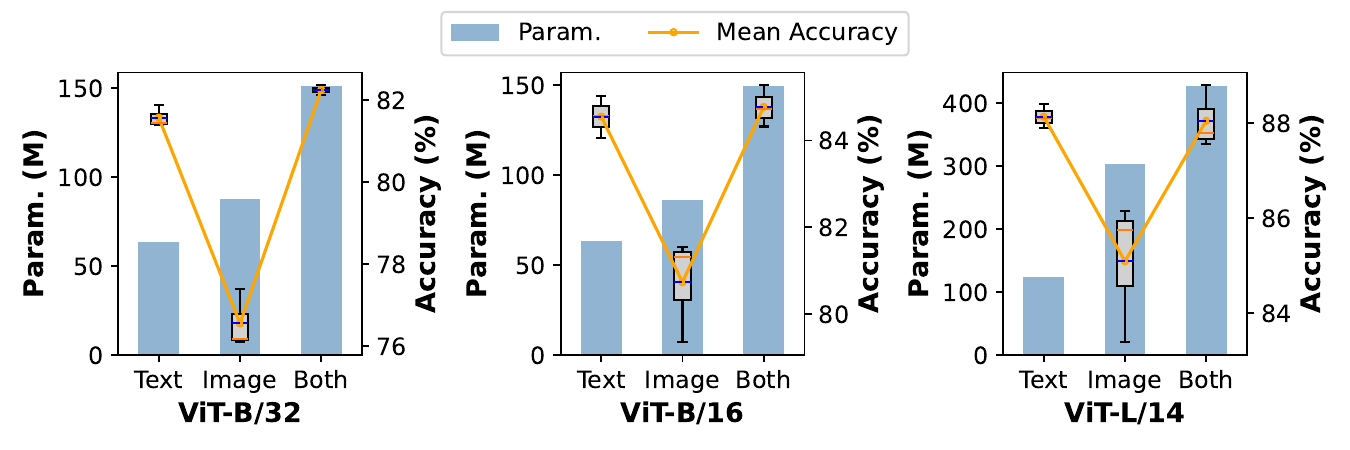}
        \vspace{-0.7cm}
        \caption{Flowers}
    \end{subfigure}
    \hfill
    \begin{subfigure}[b]{0.495\linewidth}
        \includegraphics[width=\linewidth]{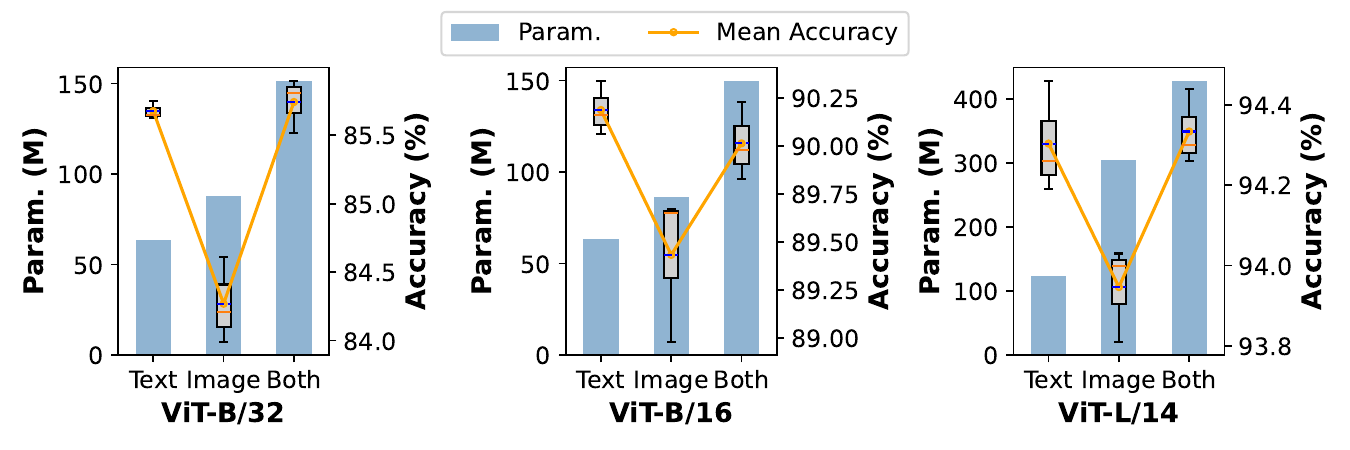}
        \vspace{-0.7cm}
        \caption{Food101}
    \end{subfigure}

    \begin{subfigure}[b]{0.495\linewidth}
        \includegraphics[width=\linewidth]{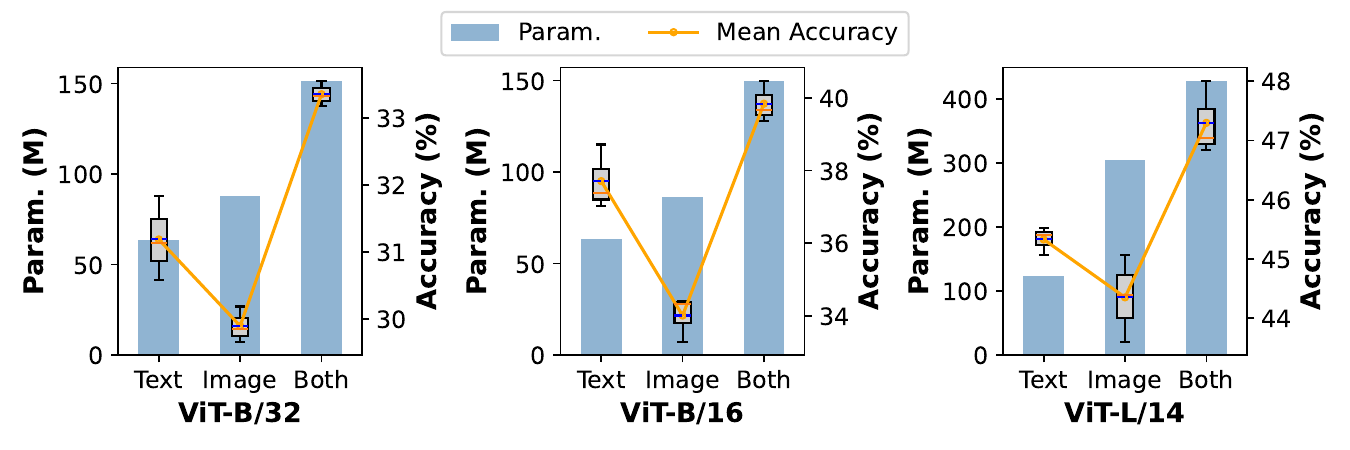}
        \vspace{-0.7cm}
        \caption{Aircraft}
    \end{subfigure}
    \hfill
    \begin{subfigure}[b]{0.495\linewidth}
        \includegraphics[width=\linewidth]{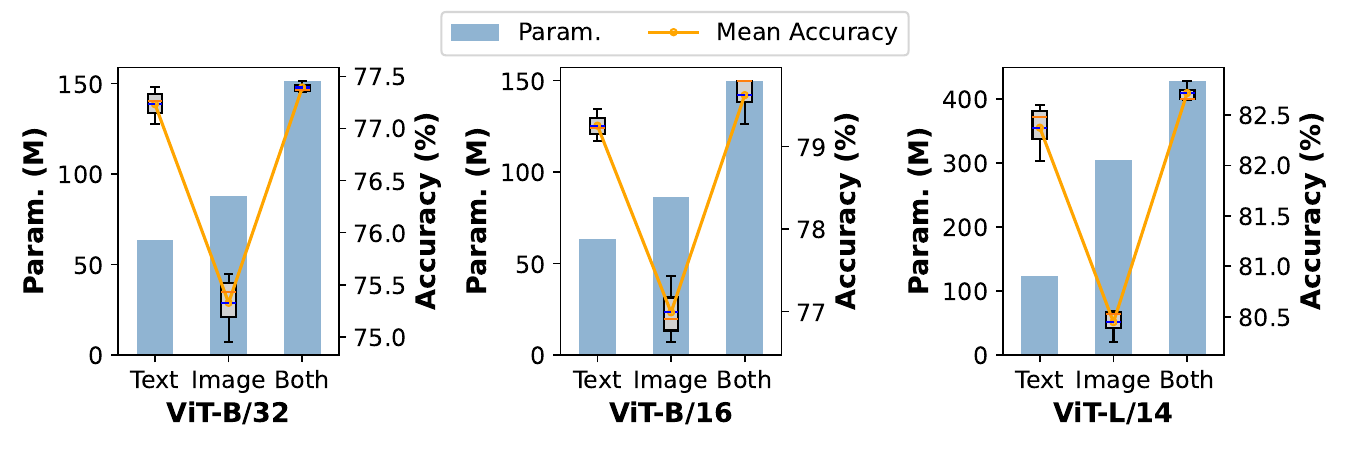}
        \vspace{-0.7cm}
        \caption{SUN397}
    \end{subfigure}

    \begin{subfigure}[b]{0.495\linewidth}
        \includegraphics[width=\linewidth]{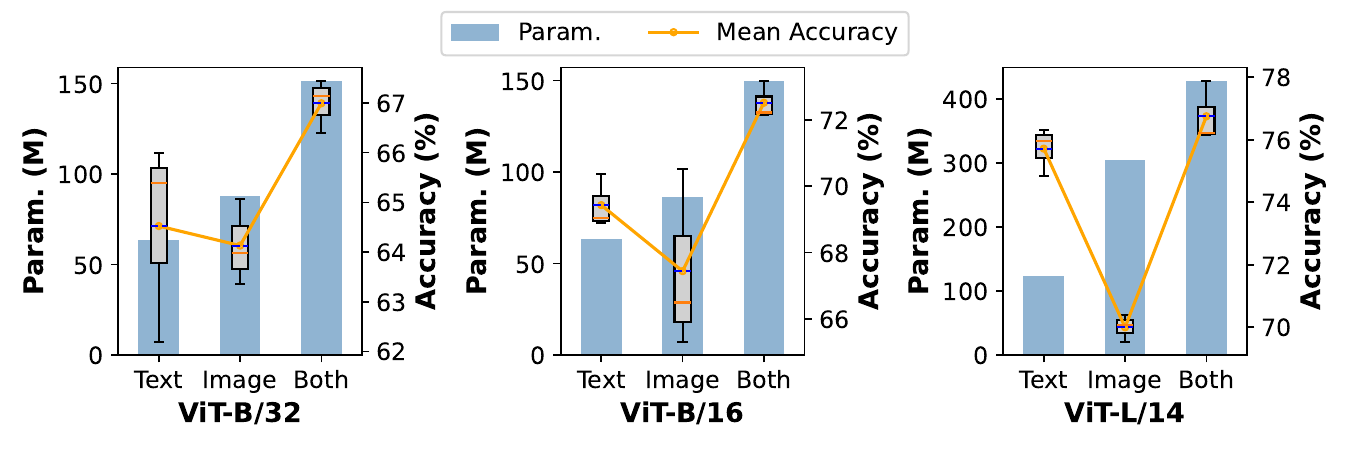}
        \vspace{-0.7cm}
        \caption{DTD}
    \end{subfigure}
    \hfill
    \begin{subfigure}[b]{0.495\linewidth}
        \includegraphics[width=\linewidth]{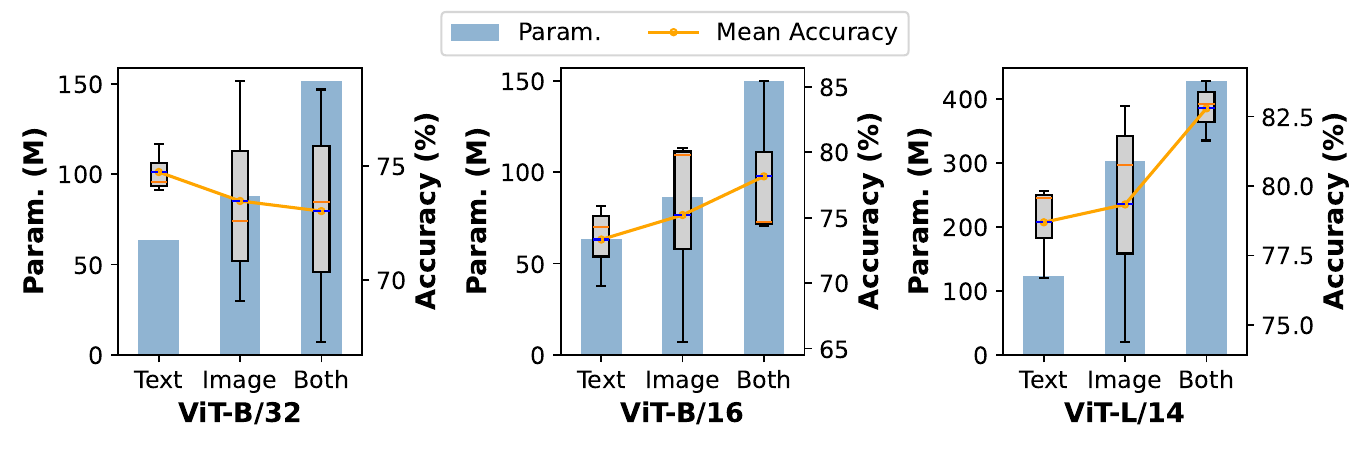}
        \vspace{-0.7cm}
        \caption{EuroSAT}
    \end{subfigure}

    \begin{subfigure}[b]{0.495\linewidth}
        \includegraphics[width=\linewidth]{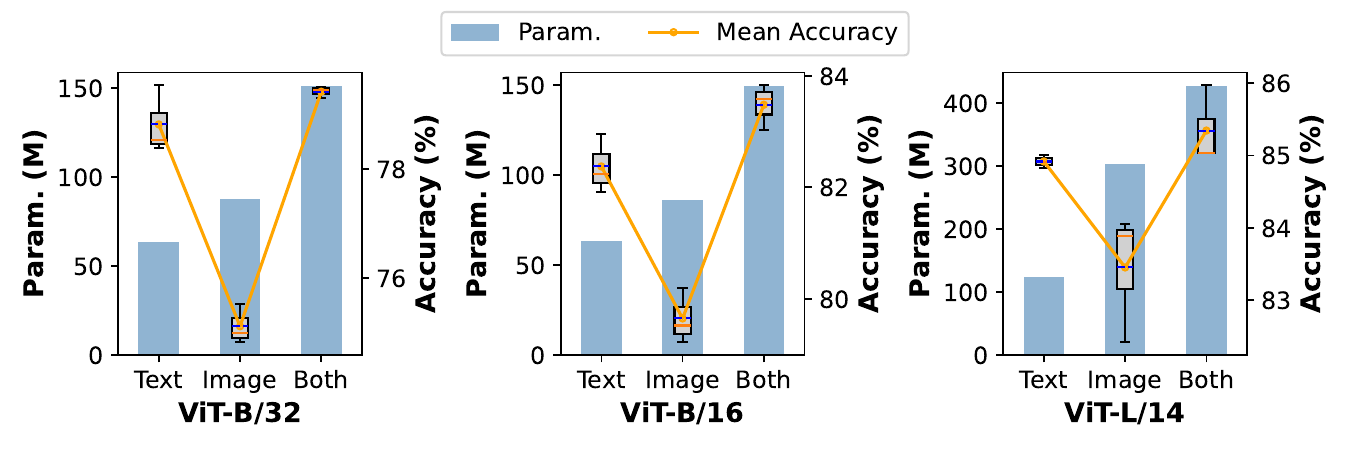}
        \vspace{-0.7cm}
        \caption{UCF101}
    \end{subfigure}

    \vspace{-0.2cm}
    \caption{The performance of symmetric (both) and asymmetric (text and image) adapters in the Base-to-Novel Generalization task across 11 datasets with three pretrained transformer-based CLIP models.}
    \label{fig:motivation_base2new}
\end{figure*}

\begin{figure*}[t]
    \centering
    
    \begin{subfigure}[b]{0.495\linewidth}
        \includegraphics[width=\linewidth]{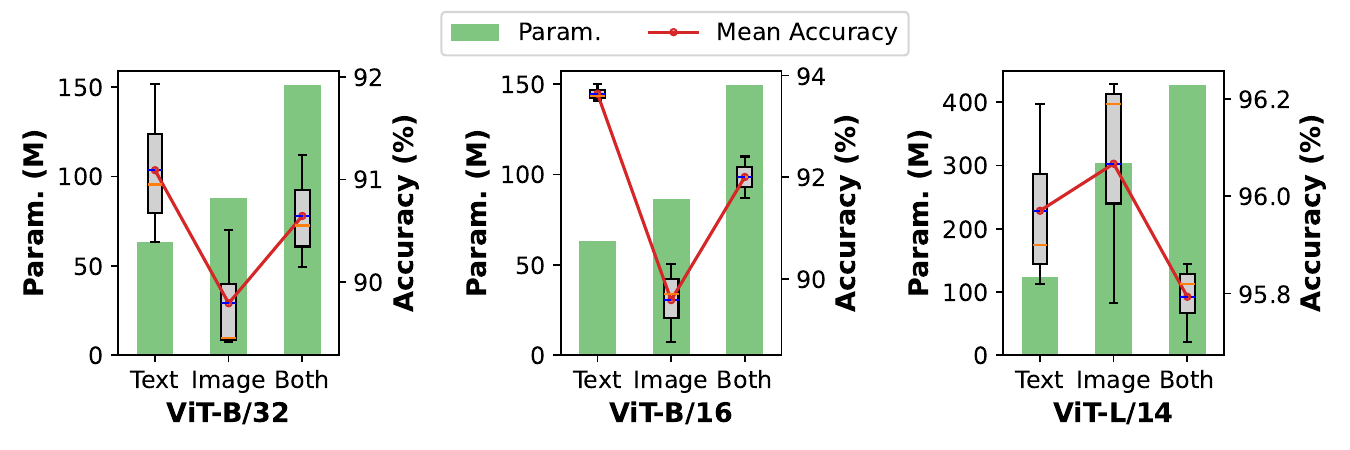}
        \vspace{-0.7cm}
        \caption{Caltech101}
    \end{subfigure}
    \hfill
    \begin{subfigure}[b]{0.495\linewidth}
        \includegraphics[width=\linewidth]{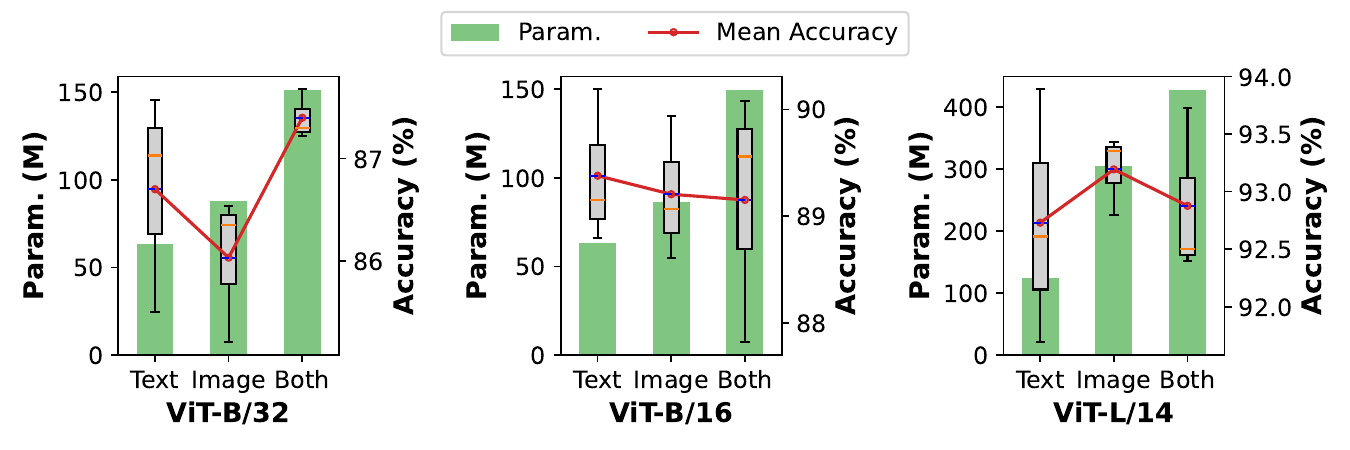}
        \vspace{-0.7cm}
        \caption{OxfordPets}
    \end{subfigure}
    \begin{subfigure}[b]{0.495\linewidth}
        \includegraphics[width=\linewidth]{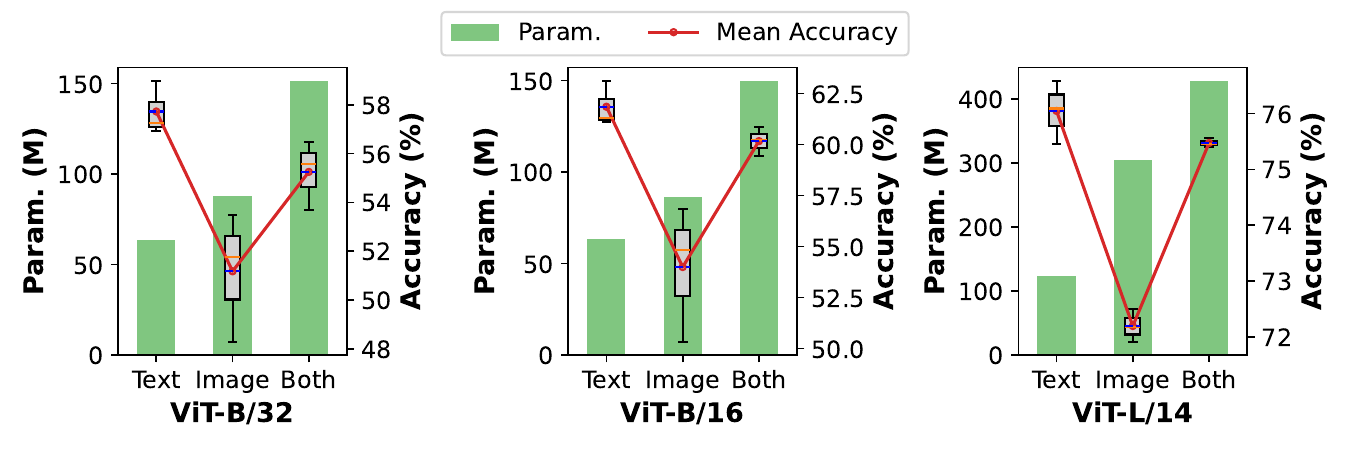}
        \vspace{-0.7cm}
        \caption{Cars}
    \end{subfigure}
    \hfill
    \begin{subfigure}[b]{0.495\linewidth}
        \includegraphics[width=\linewidth]{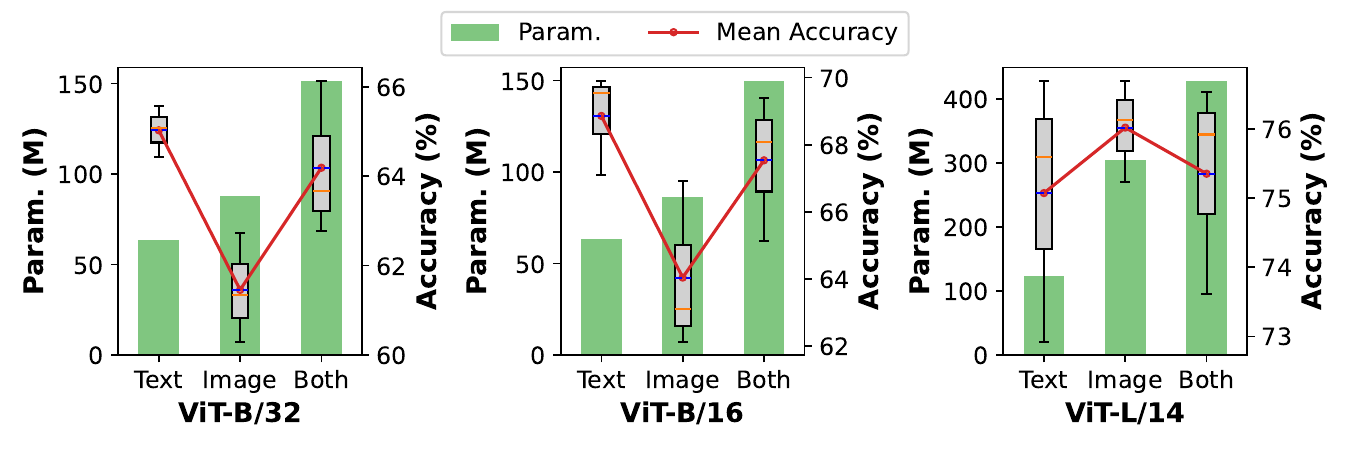}
        \vspace{-0.7cm}
        \caption{Flowers}
    \end{subfigure}
    \begin{subfigure}[b]{0.495\linewidth}
        \includegraphics[width=\linewidth]{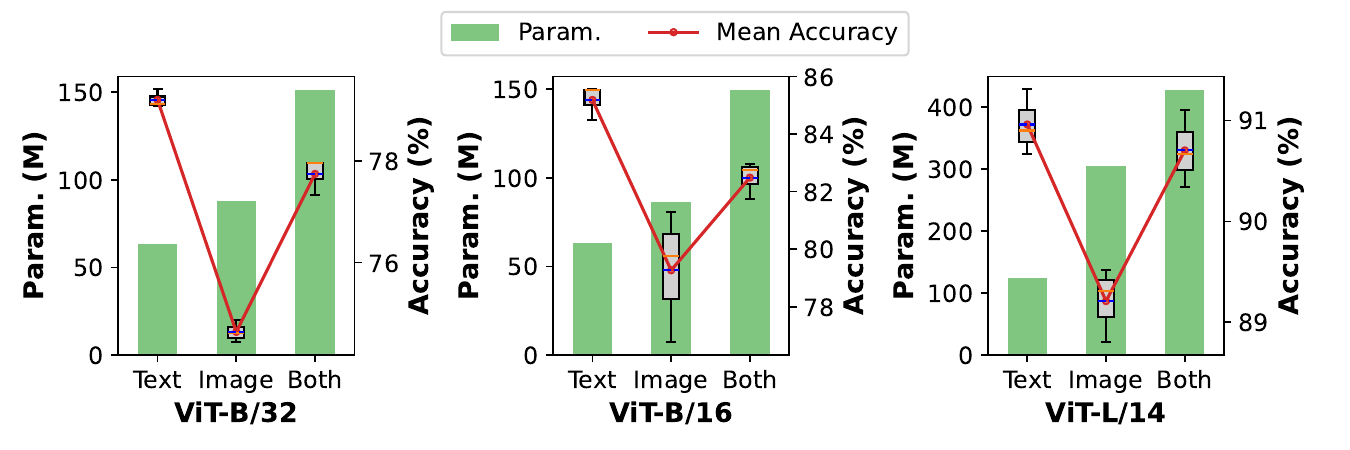}
        \vspace{-0.7cm}
        \caption{Food101}
    \end{subfigure}
    \hfill
    \begin{subfigure}[b]{0.495\linewidth}
        \includegraphics[width=\linewidth]{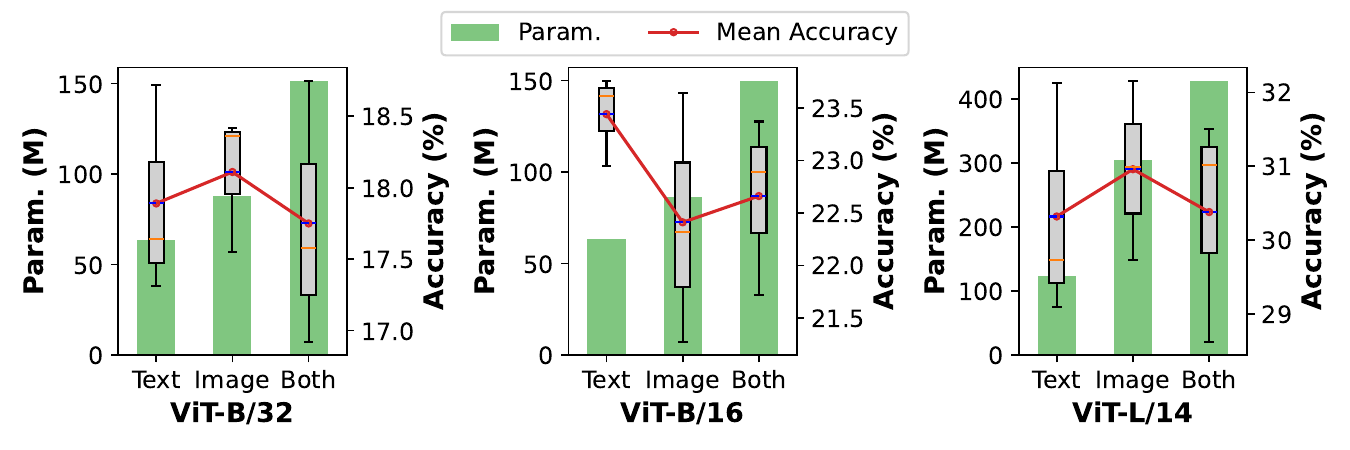}
        \vspace{-0.7cm}
        \caption{Aircraft}
    \end{subfigure}
    \begin{subfigure}[b]{0.495\linewidth}
        \includegraphics[width=\linewidth]{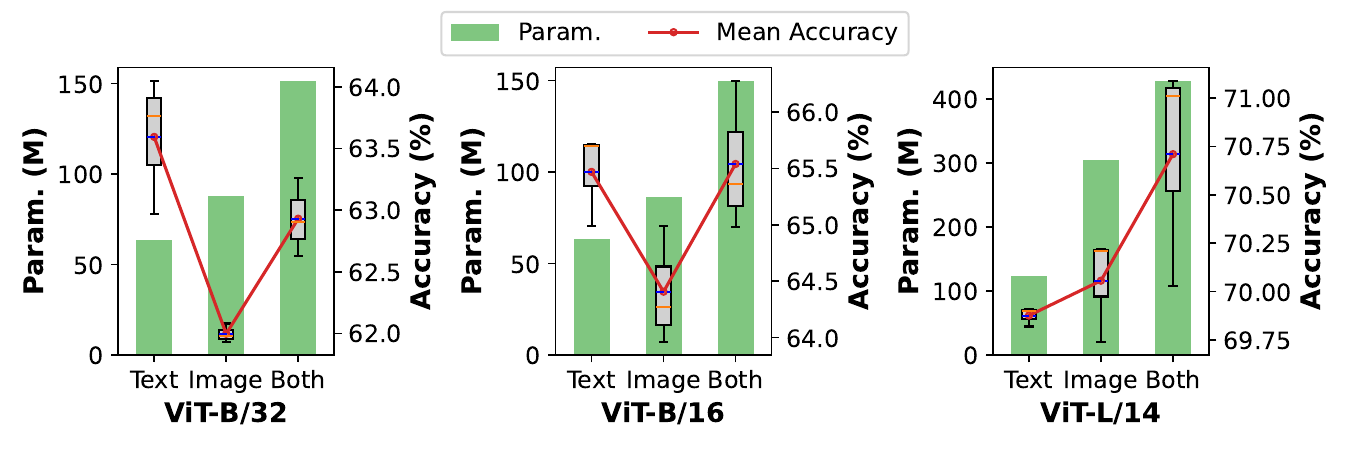}
        \vspace{-0.7cm}
        \caption{SUN397}
    \end{subfigure}
    \hfill
    \begin{subfigure}[b]{0.495\linewidth}
        \includegraphics[width=\linewidth]{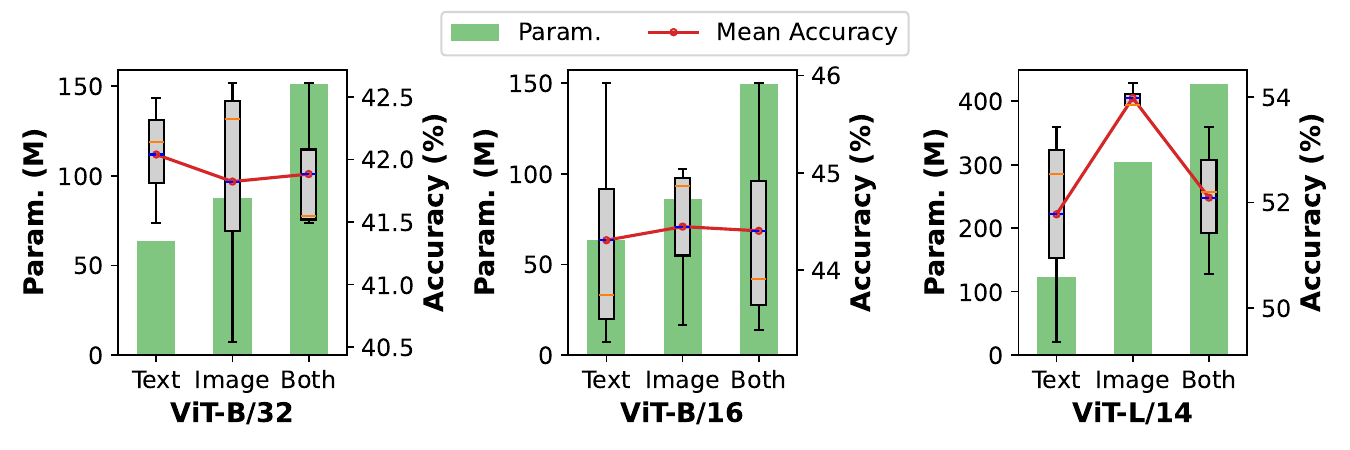}
        \vspace{-0.7cm}
        \caption{DTD}
    \end{subfigure}
    \begin{subfigure}[b]{0.495\linewidth}
        \includegraphics[width=\linewidth]{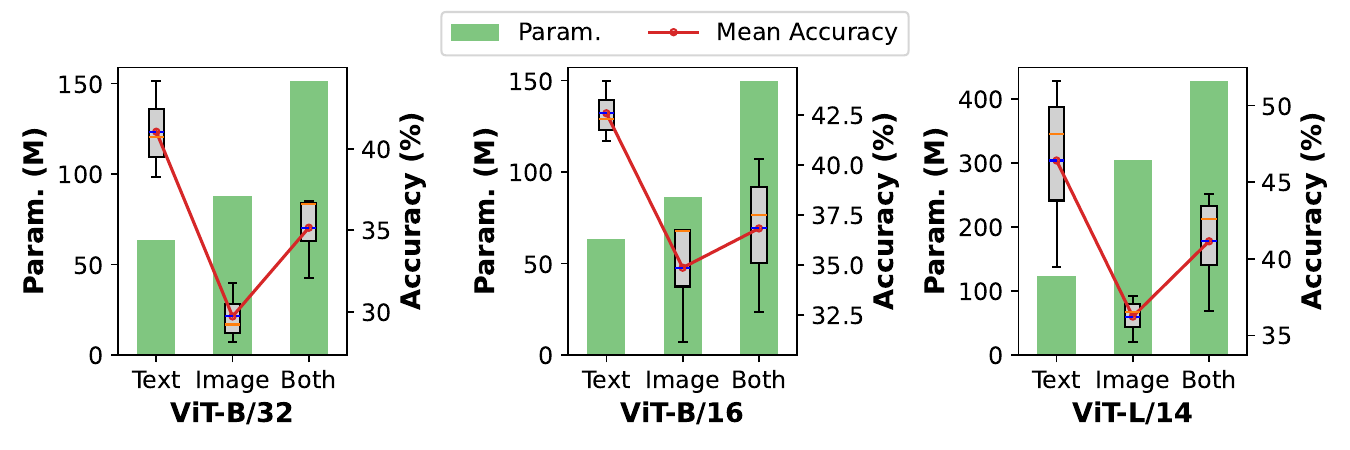}
        \vspace{-0.7cm}
        \caption{EuroSAT}
    \end{subfigure}
    \hfill
    \begin{subfigure}[b]{0.495\linewidth}
        \includegraphics[width=\linewidth]{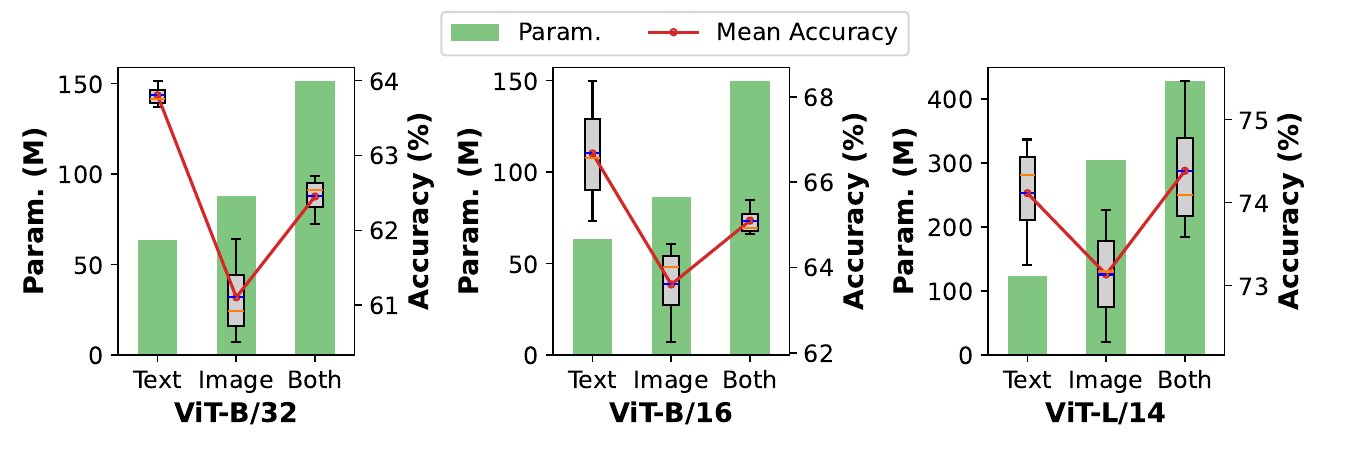}
        \vspace{-0.7cm}
        \caption{UCF101}
    \end{subfigure}

    \vspace{-0.2cm}
    \caption{The performance of symmetric (both) and asymmetric (text and image) adapters in the Cross-Dataset Evaluation task across 10 datasets with three pretrained transformer-based CLIP models.}
    \label{fig:motivation_cross_dataset}
\end{figure*}

\begin{figure*}[t!]
    \centering
    
    \begin{subfigure}[b]{0.495\linewidth}
        \includegraphics[width=\linewidth]{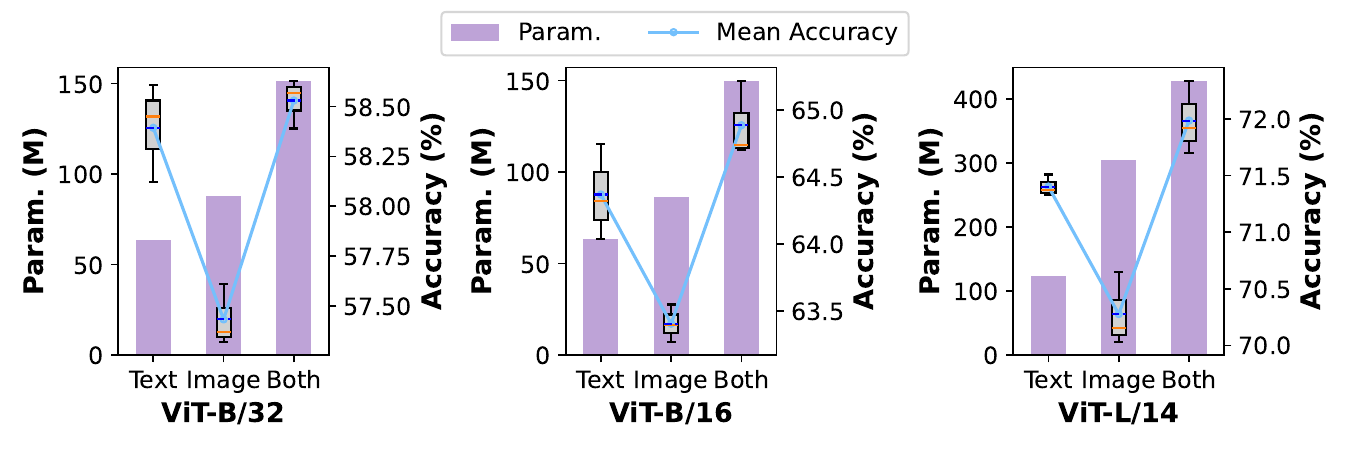}
        \vspace{-0.7cm}
        \caption{ImageNetV2}
    \end{subfigure}
    \hfill
    \begin{subfigure}[b]{0.495\linewidth}
        \includegraphics[width=\linewidth]{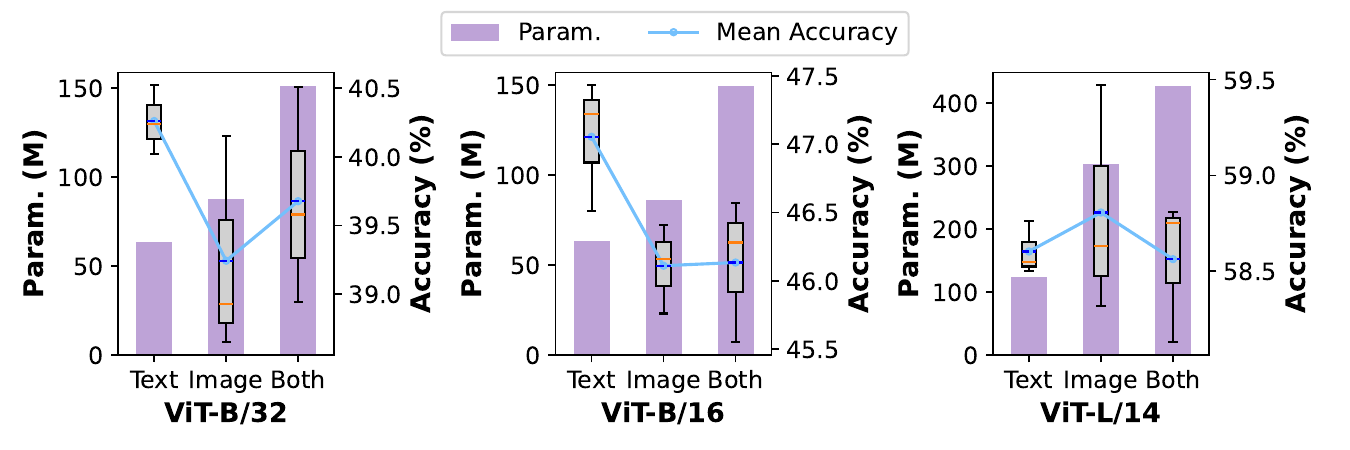}
        \vspace{-0.7cm}
        \caption{ImageNet-Sketch}
    \end{subfigure}
    \begin{subfigure}[b]{0.495\linewidth}
        \includegraphics[width=\linewidth]{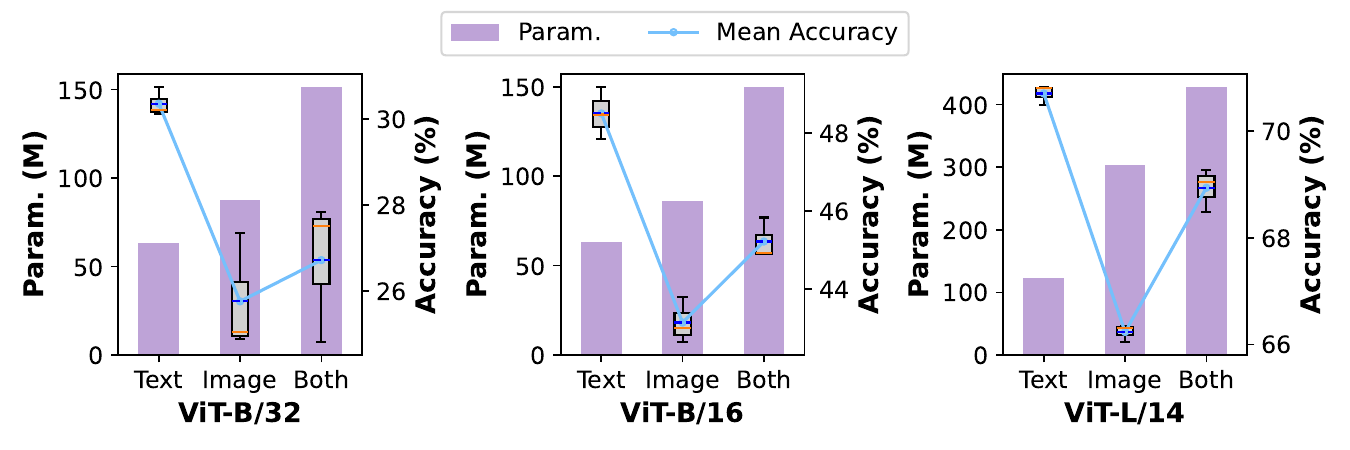}
        \vspace{-0.7cm}
        \caption{ImageNet-A}
    \end{subfigure}
    \hfill
    \begin{subfigure}[b]{0.495\linewidth}
        \includegraphics[width=\linewidth]{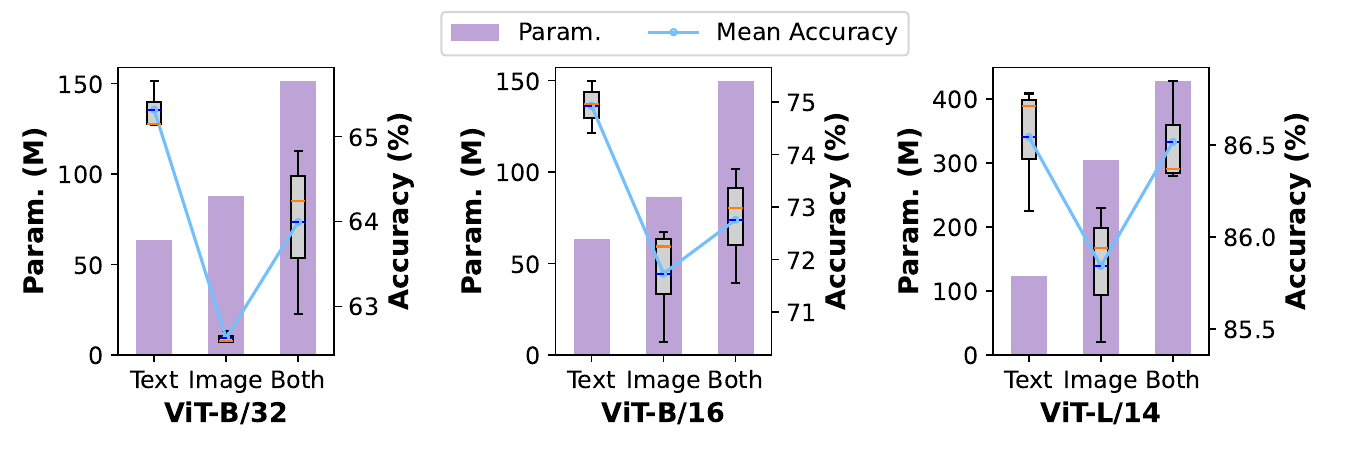}
        \vspace{-0.7cm}
        \caption{ImageNet-R}
    \end{subfigure}

    \vspace{-0.2cm}
    \caption{The performance of symmetric (both) and asymmetric (text and image) adapters in the Domain Generalization task across 4 datasets with three pretrained transformer-based CLIP models.}
    \label{fig:motivation_generalization}
\end{figure*}

\section{Additional Experimental Results}

\label{sec:additional}

In this section, we present additional experimental results mentioned in the main text, including The results of three methods (CoCoOp, CLIP-Adapter, and RPO) on the cross-dataset evaluation and domain generalization tasks (Tables~\ref{tab:cross_other} and~\ref{tab:domain_other}). Additionally, we report results for \emph{A$_3$B$_2$} and 7 leading few-shot learning methods across 11 datasets (Tables~\ref{tab:shot-base}, \ref{tab:shot-novel} and \ref{tab:shot-hm}).

\section{Computational Overhead}
\label{sec:cost}
We provide supplementary information on the training/inference time (minutes) for the Novel setting of the base-to-novel generalization task on a single A800 GPU. The comparison includes the baseline CLIP model, strong symmetric adapter-style MMA, prompt-learning-based MMRL benchmarks, and the proposed \emph{A$_3$B$_2$} on the EuroSAT dataset:

\begin{table}[htbp]
\centering
\begin{tabular}{lcc}
\toprule[1.5pt]
\midrule
Method & Train time & Inference time \\
\midrule
CLIP & -- & 2.37 \\
MMA & 10.25 & 2.96 \\
MMRL & 18.72 & 3.58 \\
\emph{A$_3$B$_2$} & 13.39 & 3.24 \\
\bottomrule[1.5pt]
\end{tabular}
\vspace{-0.3cm}
\caption{Training and inference time comparison of different methods}
\label{tab:time_comparison}
\end{table}

Table~\ref{tab:time_comparison} shows that although MMA achieves the shortest training and inference time for this task, its performance is inferior to MMRL. Additionally, on out-of-distribution tasks, MMA has more parameters than the proposed \emph{A$_3$B$_2$}. While MMRL consistently delivers strong performance across all benchmarks, it incurs higher time costs due to the complexity of training loss optimization and inference functions. In contrast, the proposed \emph{A$_3$B$_2$} strikes a balance between training and inference time while achieving the best overall performance.

\section{CNN-based CLIP Model Experiment}
\label{sec:cnn}

To our knowledge, most existing efficient transfer learning methods for VLMs have not been implemented on CNN-based CLIP models. This is because CNN-based CLIP models generally perform worse than ViT-based CLIP models with significantly fewer parameters. Therefore, applying these methods to CNN-based CLIP models is typically not cost-effective. However, to further demonstrate the applicability of the proposed method, we have included comparative experiments on the base-to-novel generalization task using a ResNet-50-based model on the EuroSAT dataset:

\begin{table*}[htbp]
\centering
\resizebox{\textwidth}{!}{%
\begin{tabular}{l|cccccccccccc}
\toprule[1.5pt]
\midrule
Setting & CLIP & CoOp & CoCoOp & KgCoOp & RPO & MaPLe & CLIP-Adapter & TCP & MMA & MMRL++ & MMRL & \emph{A$_3$B$_2$} \\
\midrule
Base  & 49.64 & 69.56 & 65.83 & 73.28 & 51.38 & 75.39 & 50.36 & 73.95 & 77.56 & 77.64 & 77.92 & 78.94 \\
Novel & 53.26 & 54.73 & 55.93 & 50.35 & 55.46 & 59.47 & 52.68 & 60.82 & 57.83 & 62.59 & 62.64 & 63.15 \\
HM    & 51.39 & 61.26 & 60.48 & 59.69 & 53.34 & 66.49 & 51.49 & 66.75 & 66.26 & 69.31 & 69.45 & 70.17 \\
\bottomrule[1.5pt]
\end{tabular}
}
\vspace{-0.3cm}
\caption{Experimental results of ResNet-50-based models.}
\label{tab:setting_comparison}
\end{table*}
Table~\ref{tab:setting_comparison} validates the robustness of the proposed \emph{A$_3$B$_2$}.

\end{document}